\documentclass[conference]{IEEEtran}

\usepackage{amsmath}
\usepackage{amsthm}
\usepackage[T1]{fontenc}
\usepackage{newtxmath}
\usepackage{diagbox}
\usepackage{slashbox}
\usepackage{graphicx}
\usepackage{multirow}
\usepackage{multicol}
\usepackage{booktabs}
\usepackage{tabularx}
\usepackage{adjustbox}
\usepackage{mathtools}
\usepackage{enumerate}
\usepackage{dblfloatfix}
\usepackage{algorithm}
\usepackage{color}
\usepackage[usenames,dvipsnames,table]{xcolor}

\usepackage{tikz}
\usepackage{tikz-qtree}
\usetikzlibrary{shapes, arrows, positioning}

\usepackage[sorting=nyt,sortcites=true]{biblatex}
\usepackage{hyperref}
\hypersetup{
    colorlinks=true,
    linkcolor=MidnightBlue,
    filecolor=magenta,      
    urlcolor=MidnightBlue,
    citecolor=MidnightBlue,
} 
\usepackage{bookmark}
\usepackage[font=small,labelfont=bf]{caption}

\newtheorem{assumption}{Assumption}

\newtheorem{theorem}{Theorem}
\newtheorem{proposition}{Proposition}
\newtheorem{corollary}{Corollary}

\definecolor{customblue}{HTML}{156082}
\definecolor{customgreen}{HTML}{196B24}
\definecolor{customorange}{HTML}{E97132}

\begin{document}

\title{{\Huge \textit{\textbf{Force Policy}}}: \\ {\huge Learning Hybrid Force-Position Control Policy under Interaction Frame for Contact-Rich Manipulation}}


\author{
\authorblockN{Hongjie Fang$^{*,1,3}$, Shirun Tang$^{*,1}$, Mingyu Mei$^{*,1,4}$, Haoxiang Qin$^3$, Zihao He$^3$, Jingjing Chen$^3$, } \authorblockN{Ying Feng$^{3,5}$, Chenxi Wang$^1$, Wanxi Liu$^{1,2}$, Zaixing He$^4$, Cewu Lu$^{1,2,3,5}$, Shiquan Wang$^{1,2,\dagger}$}  
\authorblockA{$^{1}$Noematrix\quad$^{2}$Flexiv\quad$^3$Shanghai Jiao Tong University\quad$^4$Zhejiang University\quad$^5$Shanghai Innovation Institute}  
\authorblockA{$^*$Equal Contribution \quad $^{\dagger}$Corresponding Author}
}

\makeatletter
\let\@oldmaketitle\@maketitle
\renewcommand{\@maketitle}{\@oldmaketitle
\centering\vspace{0.1cm}
\includegraphics[width=\linewidth]{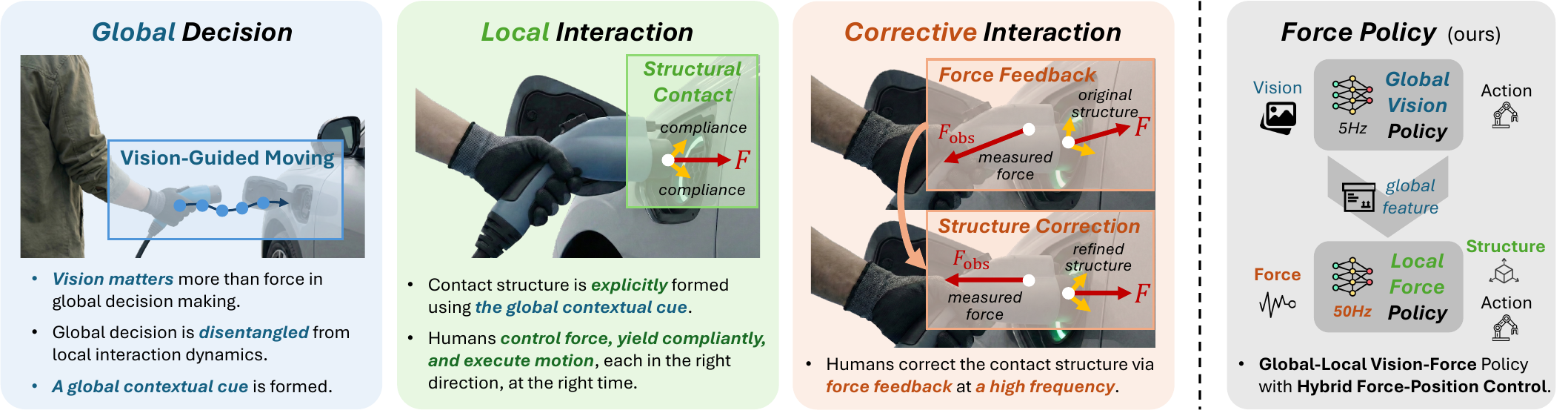}
\vspace{-0.4cm}
\captionof{figure}{\textbf{A Global-Local Vision-Force Policy Inspired by Human Interaction.} \textit{(Left)} During electric vehicle (EV) charging, humans \textbf{\color{customblue}use vision to guide global movement} and coarse alignment, and rely on \textbf{\color{customorange}force feedback} to continuously adjust the \textbf{\color{customgreen}local contact structure}. \textit{(Right)} This global-local organization motivates a vision-force policy with an explicit contact structure for stable hybrid force-position control.}
\label{fig:teaser}
\vspace{-0.3cm}
}%
\makeatother

\maketitle
\addtocounter{figure}{-1}

\begin{abstract}
Contact-rich manipulation demands human-like integration of perception and force feedback: vision should guide task progress, while high-frequency interaction control must stabilize contact under uncertainty. Existing learning-based policies often entangle these roles in a monolithic network, trading off global generalization against stable local refinement, while control-centric approaches typically assume a known task structure or learn only controller parameters rather than the structure itself. In this paper, we formalize a physically grounded interaction frame, an instantaneous local basis that decouples force regulation from motion execution, and propose a method to recover it from demonstrations. Based on this, we address both issues by proposing \textbf{\textit{Force Policy}}, a global-local vision-force policy in which a global policy guides free-space actions using vision, and upon contact, a high-frequency local policy with force feedback estimates the interaction frame and executes hybrid force-position control for stable interaction. Real-world experiments across diverse contact-rich tasks show consistent gains over strong baselines, with more robust contact establishment, more accurate force regulation, and reliable generalization to novel objects with varied geometries and physical properties, ultimately improving both contact stability and execution quality. Project website: \href{force-policy.github.io}{force-policy.github.io}.
\end{abstract}

\IEEEpeerreviewmaketitle

\section{Introduction}

Human dexterity is defined not merely by motion, but by the seamless integration of force modulation and feedback during physical interaction. Replicating this level of competence in robotic systems is the central challenge of contact-rich manipulation. Operations such as tightening a screw~\cite{jia2018survey}, peeling a vegetable~\cite{chen2024vegetable, foar}, or polishing a surface~\cite{acp, rdp} impose strict constraints where the robot must simultaneously control motion and interaction forces. Mastering contact-rich manipulation is essential for deploying robots in real-world settings, from factory assembly to everyday household tasks~\cite{suomalainen2022survey}.

A defining aspect of how humans succeed in contact-rich tasks is \textbf{a global-local organization of perception and control}~\cite{camponogara2021integration, venkadesan2008neural}, as illustrated in Fig.~\ref{fig:teaser} (left). At the global level, humans primarily use vision to decide \textit{what to do next}, \textit{where to go}, and \textit{how to align}~\cite{gharbawie2011cortical}. At the local level, once contact happens, humans rely on high-frequency force feedback to \textit{refine execution in a structured way}, applying or maintaining interaction where needed, yielding compliantly to uncertainty, and executing motion to proceed~\cite{goldring2022functional}. Functionally, this implies a clear division of roles: the global component should generalize across task variations, while the local component should guarantee stable contact. This suggests two core questions for robotics:
\begin{enumerate}
    \item[\textbf{(1)}] \textbf{Global-local organization.} How can we realize this global-local organization, so that global guidance generalizes well while local interaction remains stable?
    \item[\textbf{(2)}] \textbf{Interaction structure for control.} How can we represent the task interaction structure explicitly, so that it transfers across diverse contact skills?
\end{enumerate}

For \textbf{(1)}, many learning-based policies still adopt a monolithic design. Force signals are often appended to the observation~\cite{foar, forcevla} or used as auxiliary supervision~\cite{tavla}, but perception, planning, and contact refinement are learned jointly in a single end-to-end network. This creates an inherent trade-off: entangling local refinement with global perception makes it difficult to maintain a dedicated high-frequency interaction loop while improving global generalization by scaling up models. RDP~\cite{rdp} is one of the few works that explicitly introduces a slow-fast policy design, with fast reactive corrections handled by a force-aware tokenizer. However, its slow policy depends on the fast tokenizer rather than treating reactivity as a plug-and-play component, which limits reuse with arbitrary slow policies and complicates upgrading the slow policy due to the dependency of the fast reactive part.

For \textbf{(2)}, a substantial body of prior work approaches contact-rich manipulation through the control interface. Classical compliant control~\cite{mason_hybrid} provides principled tools such as hybrid position/force control~\cite{raibert_hybrid} and impedance control~\cite{hogan1985impedance}. Conventional approaches apply these tools through task-specific environment/contact modeling and manual parameter tuning to match a pre-defined interaction structure. Learning-based variants aim to reduce manual tuning by predicting controller parameters from data, such as admittance gains~\cite{equicontact, tacdiffusion}, or contact forces~\cite{conkey2019learning, forcemimic}. However, the interaction structure is usually assumed or implicit, and learning is applied mainly to tuning parameters rather than to representing the structure itself. ACP~\cite{acp} is a partial exception by predicting a stiffness matrix, but its directional formulation is insufficient for multi-axis constrained interactions such as peg insertion.

In this paper, we answer both questions by making the missing structure explicit and building a global-local vision-force policy around it. For \textbf{(2)}, we formalize the interaction frame, an instantaneous local basis that captures task-relevant interaction structure, from a physical perspective, and we present a principled approach to recover it directly from non-ideal real-world demonstrations. For \textbf{(1)}, we propose \textbf{\textit{Force Policy}}, a global-local vision-force policy with phase-wise authority, as shown in Fig.~\ref{fig:teaser} (right). A global vision policy drives task progress, while a high-frequency force module estimates the interaction frame and executes hybrid force-position control during contact. This decoupling preserves task-level planning while ensuring stable contact and reliable execution under uncertainty. Extensive real-world experiments show improved robustness and force regulation across diverse contact-rich tasks, while retaining strong generalization to novel objects with varied geometries and physical properties. Overall, \textbf{\textit{Force Policy}} combines the generalization of visuomotor policies with the precision of force control, enabling reliable contact-rich manipulation. 
\section{Related Works}

\subsection{Contact-Rich Manipulation}
Compared to free-space motion, contact-rich manipulation requires precise force regulation and rapid adaptation to contact constraints~\cite{whitney1987historical, suomalainen2022survey}. Uncertainties in contact geometry, friction, and material compliance make interaction dynamics difficult to model reliably~\cite{mason_hybrid, de1988compliant}. Classical robotics addresses these challenges with compliant control frameworks, including impedance control~\cite{hogan1985impedance, buchli2011learning}, admittance control~\cite{admittance_control, keemink2018admittance, kronander2016stability}, and hybrid force-position control~\cite{raibert_hybrid, khatib2003unified, chiaverini2002parallel}. Nonetheless, these approaches often rely on hand-tuned parameters and reasonably accurate environment assumptions, which can limit robustness in unstructured settings~\cite{villani2016force}.

Learning-based methods have recently improved contact-rich manipulation by learning interaction strategies directly from data. Reinforcement learning~\cite{kalakrishnan2011learning, levine2015learning, forge, mimictouch} can acquire complex contact behaviors through exploration, but frequently struggles with sim-to-real transfer due to mismatches in perception and contact dynamics~\cite{chukwurah2024sim}. Hence, imitation learning has increasingly leveraged richer sensing to better ground interaction, incorporating modalities like audio~\cite{maniwav, levine2015learning, seehearfeel, playtothescore,wang2025sound}, tactile~\cite{huang20243d, seehearfeel, playtothescore, eyesight_hand, dexop}, and force/torque signals~\cite{buamanee2024bi, acp, kamijo2024learning, forcemimic, tacdiffusion, zhou2025admittance}, to handle contact-rich manipulations.

\subsection{Force-Aware Manipulation Policies}

Recent advances have explored different strategies to incorporate force sensing into learning-based manipulation. A common strategy treats force as an auxiliary observation, fusing it with visual features to condition motion prediction. FoAR~\cite{foar} dynamically weights visual and force via contact-phase prediction. RDP~\cite{rdp} couples a visual latent diffusion policy with a high-frequency tactile-aware tokenizer for reactive execution. ForceVLA~\cite{forcevla} integrates force through a mixture-of-experts representation, while TA-VLA~\cite{tavla} adds torque prediction as an auxiliary objective to promote physically grounded features. 

A distinct body of work learns imitation policies within compliant control frameworks, spanning variable impedance and admittance formulations~\cite{acp, tacdiffusion, kamijo2024learning, dipcom, equicontact, zhou2025admittance}. Even when driven by predicted wrenches, these controllers often tie force realization to motion through a single compliance model, so what the policy should imitate can drift with contact geometry and local stiffness. In contrast, hybrid force-position control provides a more interpretable interface by explicitly separating constrained and free directions: it lets the policy express ``\textit{what should be enforced}'' (force objectives in constraint directions) and ``\textit{what should be achieved}'' (position objectives in free directions), yielding clearer credit assignment and more stable learning across contact transitions. While \cite{forcemimic} adopts a hybrid scaffold, it applies force control only along a single direction, capturing only a limited form of this factorization.

\subsection{Control Structure Discovery from Demonstrations}

The problem of inferring control structures from demonstration data has progressed from analytical geometric modeling~\cite{mason_hybrid, bruyninckx1996specification, bruyninckx1995kinematic} to data-driven learning approaches~\cite{ureche2015task, kober2015learning, subramani2018inferring}. Across this spectrum, the goal is to recover the underlying control structure of a task, typically a decomposition that enables compliant manipulation. Because this structure is induced by task constraints, inferring the control structure is essentially equivalent to inferring the constraints from demonstration data.

Task frame formalization (TFF)~\cite{mason_hybrid, bruyninckx1996specification} provides the canonical force-motion subspace decomposition, but typically assumes explicit geometry and rigid contact. Subsequent works infer geometric constraints from demonstrations via twist-wrench analysis~\cite{suomalainen2017geometric, suomalainen2016learning, suomalainen2021imitation, suomalainen2018learning} or by fitting a set of pre-defined constraint models~\cite{subramani2018inferring}, with refinements using interactive perception~\cite{li2022learning, perez2017c} or augmentation~\cite{li2023augmentation}. Some works~\cite{conkey2019learning, forcemimic} exploit contact force direction as the force control axis in hybrid force-position control~\cite{raibert_hybrid}. Recent attempts improve robustness by aggregating frame estimates from multiple model-based fits using statistical fusion~\cite{analytic_if}, or by estimating task-aligned frames via kinematics~\cite{kinematic_only} or power-based objectives~\cite{interaction_frame}.

In this work, we generalize the definition of task frame into a physically-grounded interaction frame and propose compact approximations based on energy dissipation, enabling reliable recovery from data for different contact-rich tasks.

\section{Theoretical Formulation}

\subsection{Interaction Frame: A Physical Perspective}

In classical TFF~\cite{bruyninckx1996specification}, the interaction frame is prescribed based on known geometric models. However, in unstructured environments where geometry is unknown or uncertain, such a prescription is infeasible. To bridge this gap, we reformulate \textbf{interaction frame (IF)}, described in~\cite{interaction_frame}, \textit{purely from the physical response}, using stiffness as a proxy to estimate the local geometry. We focus on tasks maintaining \textit{topologically invariant contact}, excluding irreversible changes like fracture and plastic deformations. For these \textit{locally conservative} contacts, the environmental stiffness is \textit{geometry-induced}, as resistance arises directly from physical boundaries. This inherent physical causality implies that the principal axes of stiffness structurally align with the local surface geometry, allowing us to recover the geometric frame from physical observations.

Formally, we model the local environment response at the interaction point $p_I$ by \textit{symmetric} stiffness $\mathbf K_{\mathrm{env}}(p_I)$. We spectrally decompose $\mathbf K_{\mathrm{env}}(p_I)$ into principal stiffnesses $\lambda_i$ and axes $\mathbf q_i$, partitioning the interaction space into the \textbf{constraint subspace} $\mathcal U(p_I)$ with high stiffness $\lambda_i \gg 0$ and \textbf{admissible-motion subspace} $\mathcal T(p_I)$ with negligible stiffness $\lambda_i \approx 0$. 

Crucially, physical interaction is \textit{goal-directed}. We introduce the task intent $\{\boldsymbol{\xi}^*(p_I), \boldsymbol{\mathcal W}^*(p_I)\}$ to represent the \textit{driving factors} of the interaction --- specifically, the intended motion twist $\boldsymbol{\xi}^*(p_I)$ and interaction wrench $\boldsymbol{\mathcal W}^*(p_I)$. This formulation distinguishes our approach from TFF, where ideal variables in~\cite{bruyninckx1996specification} describe the resultant kinematic constraints without differentiating between causal actuation and environmental reaction. In our framework, \textbf{the intent represents the active cause compatible with the physical structure}: the motion intent $\boldsymbol{\xi}^*(p_I)$ acts as the driver within the admissible-motion subspace $\mathcal T(p_I)$, while the wrench intent $\boldsymbol{\mathcal W}^*(p_I)$ acts as the driver against the constraint subspace $\mathcal U(p_I)$. We thus define the IF $\Sigma(p_I)$ by anchoring the spectral axes derived from $\mathbf K_{\mathrm{env}}(p_I)$ to these driving intents:
\begin{equation}
\Sigma(p_I)\triangleq \Psi\left(\mathbf K_{\mathrm{env}}(p_I),\boldsymbol{\xi}^*(p_I), \boldsymbol{\mathcal W}^*(p_I)\right).
\end{equation}
Specifically, we set the $z$-axis of IF to the direction of the dominant wrench component in $\boldsymbol{\mathcal W}^*$. We then define the $x$-axis by projecting the motion intent $\boldsymbol{\xi}^*$ onto the plane orthogonal to $z$. In degenerate cases where this projection is zero, \textit{e.g.}, static holding with negligible motion or screw driving with collinear motion, the interaction is effectively transversely \textit{isotropic} in the constraint plane, so the rotation about $z$ is physically irrelevant. We therefore default $\Psi$ to a canonical reference and complete the frame by the right-hand rule.

To illustrate this formulation in practice, Fig.~\ref{fig:example} visualizes the derived IF for a representative set of contact-rich tasks. This physics-aware formulation retains the orthogonal decomposition power of TFF but replaces geometric priors with environmental compliance, making it intrinsically suitable for unstructured tasks. Having defined the frame theoretically, we next address the practical challenge of recovering the IF directly from interaction data.

\begin{figure}
    \centering
    \includegraphics[width=0.87\linewidth]{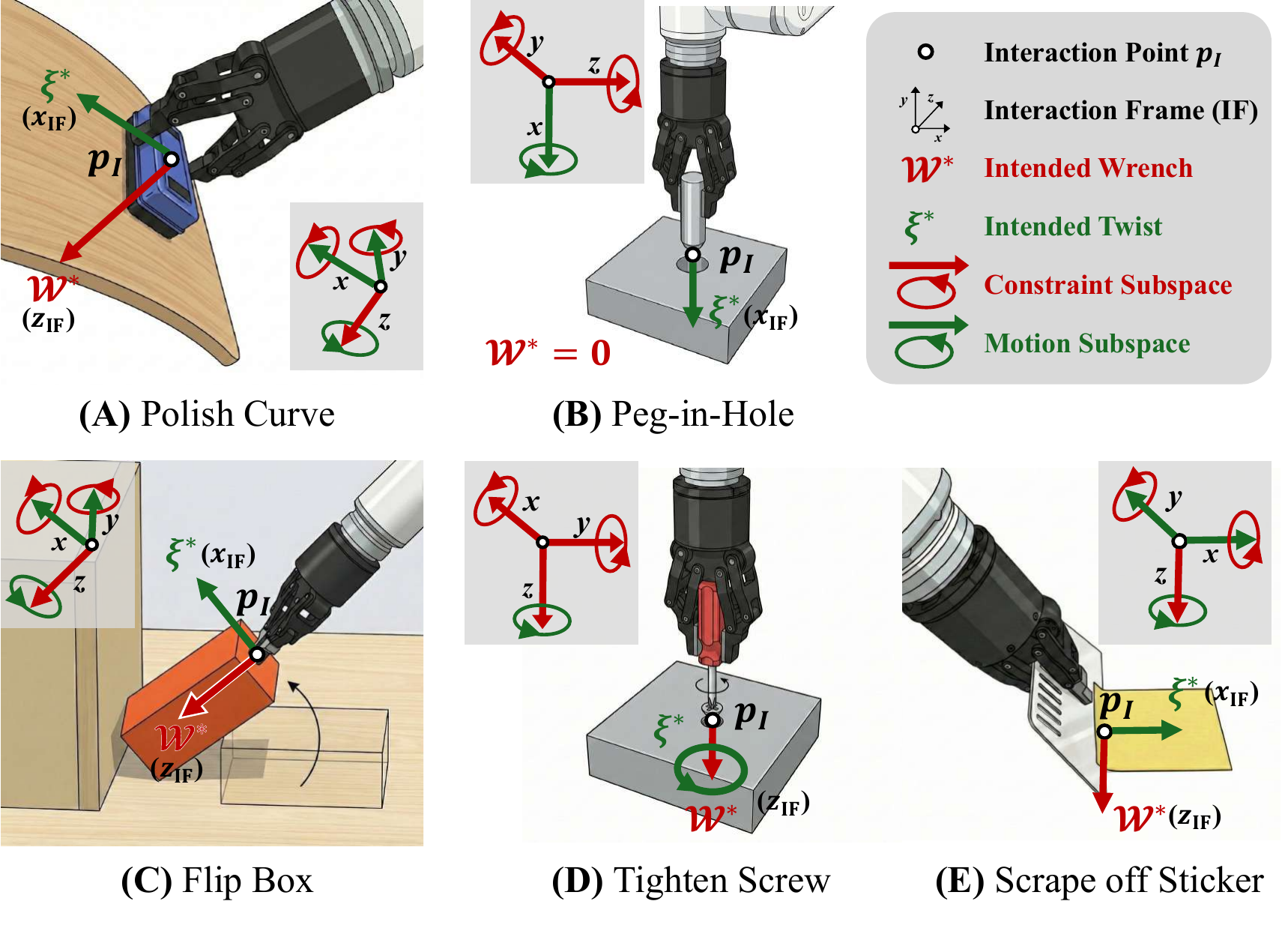}
    \caption{\textbf{Interaction Frame for Example Contact-Rich Tasks.}}
    \label{fig:example}
    \vspace{-0.5cm}
\end{figure}

\subsection{Recovering Interaction Frame from Interaction Signals}\label{sec:recovering_if}

We aim to recover IF directly from the observed interaction signals $\boldsymbol{\xi}$ and $\boldsymbol{\mathcal{W}}$. Assume that the gravity and inertial forces are compensated from the observed wrench. In the following, we omit the dependence of $p_I$ in the notation for brevity. 

Under ideal \textit{rigid and frictionless} contact, the observed twist and wrench perfectly align with the task intent, \textit{i.e.}, $\boldsymbol{\xi} = \boldsymbol{\xi}^*$ and $\boldsymbol{\mathcal W} = \boldsymbol{\mathcal W}^*$, yielding zero power: $\boldsymbol{P} = \boldsymbol{\mathcal{W}}^\top \boldsymbol{\xi} = \textbf{0}$. In the non-ideal situations, the intended twist $\boldsymbol{\xi}^*$ may induce a parasitic wrench $\boldsymbol{\mathcal{W}}_c$ along its direction. Since humans tend to align with the principal axes of the environmental stiffness $\mathbf{K}_\text{env}$ in expert demonstrations, the intended wrench $\boldsymbol{\mathcal W}^*$ may induce a parasitic twist $\boldsymbol{\xi}_c$ along its direction. Hence,
\begin{equation}
    \boldsymbol{P} = (\boldsymbol{\mathcal{W}}^* + \boldsymbol{\mathcal{W}}_c)^\top (\boldsymbol{\xi}^* + \boldsymbol{\xi}_c) = \boldsymbol{\mathcal{W}}_c^\top \boldsymbol{\xi}^* + \boldsymbol{\mathcal{W}}^{*\top} \boldsymbol{\xi}_c.
\end{equation}
The two terms correspond to distinct power sources: 
$\boldsymbol{\mathcal{W}}_c^\top \boldsymbol{\xi}^*$ acts in $\mathcal{T}$ as a \textbf{dissipative residual}, arising from frictional or viscous losses along the intended motion direction, 
while $\boldsymbol{\mathcal{W}}^{*\top} \boldsymbol{\xi}_c$ acts in $\mathcal{U}$ as a \textbf{structural residual}, resulting from environmental stiffness or contact constraints such as compliant deflections along the intended wrench direction. Consequently:
\begin{enumerate}[(1)]
    \item If the structural residual dominates, we obtain $\boldsymbol{\mathcal W}^* \approx \boldsymbol{\mathcal W}$, and the twist can be orthogonalized against the wrench:
    \begin{equation}
        \boldsymbol{\xi}^* = \boldsymbol{\xi} - \mathrm{Proj}_{\boldsymbol{\mathcal W}^*}(\boldsymbol{\xi}) \approx \boldsymbol{\xi} - \mathrm{Proj}_{\boldsymbol{\mathcal W}}(\boldsymbol{\xi}).
    \end{equation}

    \item If the dissipative residual dominates, we obtain $\boldsymbol{\xi}^* \approx \boldsymbol{\xi}$, and the wrench can be orthogonalized against the twist:
    \begin{equation}
        \boldsymbol{\mathcal W}^* = \boldsymbol{\mathcal W} - \mathrm{Proj}_{\boldsymbol{\xi}^*}(\boldsymbol{\mathcal W}) \approx \boldsymbol{\mathcal W} - \mathrm{Proj}_{\boldsymbol{\xi}}(\boldsymbol{\mathcal W}).
    \end{equation}
\end{enumerate}

This framework contextualizes the limitations of prior art: \cite{conkey2019learning, forcemimic} exclusively address structural dominance, effectively ignoring the dissipative dominance regime critical for friction-heavy tasks. \cite{analytic_if} relies on statistical averaging, which lacks physical interpretability and obscures the intrinsic environmental properties. \cite{interaction_frame} directly performs power minimization, which may converge to spurious frames that satisfy orthogonality but fail to capture the true task structure.

\section{Method}

In this section, we first present how to recover the control structure of contact-rich tasks from demonstrations (\S\ref{sec:control_structure}). We then introduce \textbf{\textit{Force Policy}}, a global-local vision-force policy inspired by the human organization of perception and control (\S\ref{sec:forcepolicy}). Finally, we present a dual-policy asynchronous scheduler that manages dual-frequency policy execution during deployment to ensure smooth trajectories (\S\ref{sec:scheduler}).

Let $\tau = (I_{0:T}, s_{0:T}, e_{0:T}, \boldsymbol{\xi}_{0:T}, \boldsymbol{\mathcal{W}}_{0:T})$ denote a demonstration of task $\mathcal{T}$ with horizon $T$, where $I_t$ is the visual observation, $s_t$ the end-effector pose, and $e_t$ the end-effector state at timestep $t$. The robot twist $\boldsymbol{\xi}_t\in \mathfrak{se}(3)$ is composed of its linear velocity $\boldsymbol{v}_t$ and angular velocity $\boldsymbol{\omega}_t$, \textit{i.e.}, $\boldsymbol{\xi}_t = [\boldsymbol{v}_t^\top, \boldsymbol{\omega}_t^\top]^\top$, while the wrench $\boldsymbol{\mathcal{W}}_t  \in \mathfrak{se}^*(3)$ is composed of its force $\boldsymbol{f}_t$ and moment $\boldsymbol{m}_t$, \textit{i.e.}, $\boldsymbol{\mathcal{W}}_t = [\boldsymbol{f}_t^\top, \boldsymbol{m}_t^\top]^\top$. We assume $\boldsymbol{\mathcal{W}}_t$ is gravity-compensated, and inertial effects are negligible given the smooth motion profile of human demonstrations.

\subsection{Control Structure of Contact-Rich Tasks}\label{sec:control_structure}

\textbf{Interaction Frame Identification.} Consider a $\Delta t$ patch of a demonstration starting from timestep $t$. In this work, we do not address the estimation of the exact contact point $p_I$; instead, we assume the IF origin is anchored at the end-effector position. We aim to recover the IF orientation from visual observations $I_{t:t+\Delta t}$, measured signals $\boldsymbol{\xi}_{t:t+\Delta t}$ and $\boldsymbol{\mathcal{W}}_{t:t+\Delta t}$, together with the task description $\mathcal{T}$. Inferring the ideal interaction frame directly from these signals is ill-posed, as identical local power exchange patterns may arise from different physical mechanisms, such as structural stiffness or surface friction. We therefore adopt an adaptive approximation strategy: we prompt Gemini~3~Pro~\cite{gemini3pro} with the initial visual context $I_t$ and task description $\mathcal{T}$ to classify the dominant power source (dissipative residual or structural residual) based on high-level semantics, and apply the corresponding reconstruction formulation detailed in \S\ref{sec:recovering_if}.

\textbf{Task Classification and Control Signal Generation.} Given IF, where the $x$-axis denotes the intended motion twist direction and the $z$-axis denotes the intended interaction wrench direction, we classify interaction patches into four canonical task modes based on the \emph{IF-aligned signal characteristics} of twist $\boldsymbol{\xi}$ and wrench $\boldsymbol{\mathcal W}$:

\begin{itemize}
    \item \textbf{Free}: 
    No sustained contact. The interaction wrench is negligible, \textit{i.e.}, $\|\boldsymbol{\mathcal W}\| \approx 0$.

    \item \textbf{Surface}: 
    Contact with a surface imposing a single dominant normal constraint, \textit{i.e.}, $\|\boldsymbol{f}_z\| \gg \|\boldsymbol{f}_{xy}\|$.

    \item \textbf{Insertion}: 
    Highly-constrained insertions with strong force anisotropy from frictions, \textit{i.e.}, $\|\boldsymbol{f}_x\| \gg \|\boldsymbol{f}_{yz}\|$.

    \item \textbf{Rotation}: Rotation-dominated interaction like fastening screws, \textit{i.e.}, $\|\boldsymbol{\omega}_z\| \gg 0$, and $\|\boldsymbol{f}_z\| \gg 0$.
\end{itemize}

This categorization maps directly to the hybrid control structure via a selection mask $S \in \{0,1\}^6$~\cite{mason_hybrid}, where $S_i=1$ denotes force control and $S_i=0$ denotes position control. By configuring $S$ to align with the spectral axes of each mode, we enable the requisite motion while regulating contact forces. The specific control parameterization is detailed in Table~\ref{tab:control_structure}. 

\begin{table}[h]
\begin{center}
\caption{\textbf{Control Structure Selection based on Task Modes.} The IF $x$-axis aligns with the primary motion/insertion direction, while the $z$-axis aligns with the dominant normal.}
\label{tab:control_structure}
\begin{tabular}{@{}cccc@{}}
\toprule
\textbf{Interaction} & \textbf{Task Mode} & \textbf{Selection Mask $S$} & \textbf{Reference Wrench} \\
\toprule
Free & Free & $[0, 0, 0, 0, 0, 0]$ & $[-, -, -, -, -, -]$ \\
\midrule
\multirow{3}{*}{Contact} 
  & Surface & $[0,0,1,0,0,0]$ & $[-,-,\boldsymbol{f}_z,-,-,-]$ \\
  & Insertion & $[0,1,1,0,1,1]$* & $[-,0,0,-,0,0]$ \\
  & Rotation & $[0,0,1,0,0,0]$* & $[-,-,\boldsymbol{f}_z,-,-,-]$ \\
\bottomrule
\end{tabular}%
\end{center}
{\footnotesize * In practice, for insertion tasks, both axial translation and rotation about the insertion $x$-axis are theoretically inside $\mathcal{T}$, whereas for screw (pure rotation) tasks, only the axial rotation around $z$-axis belongs to the $\mathcal{T}$. Nevertheless, force and torque control are often applied to these degrees of freedom to cope with high friction and prevent jamming.}
\vspace{-0.3cm}
\end{table}

\begin{figure*}
    \centering
    \includegraphics[width=\linewidth]{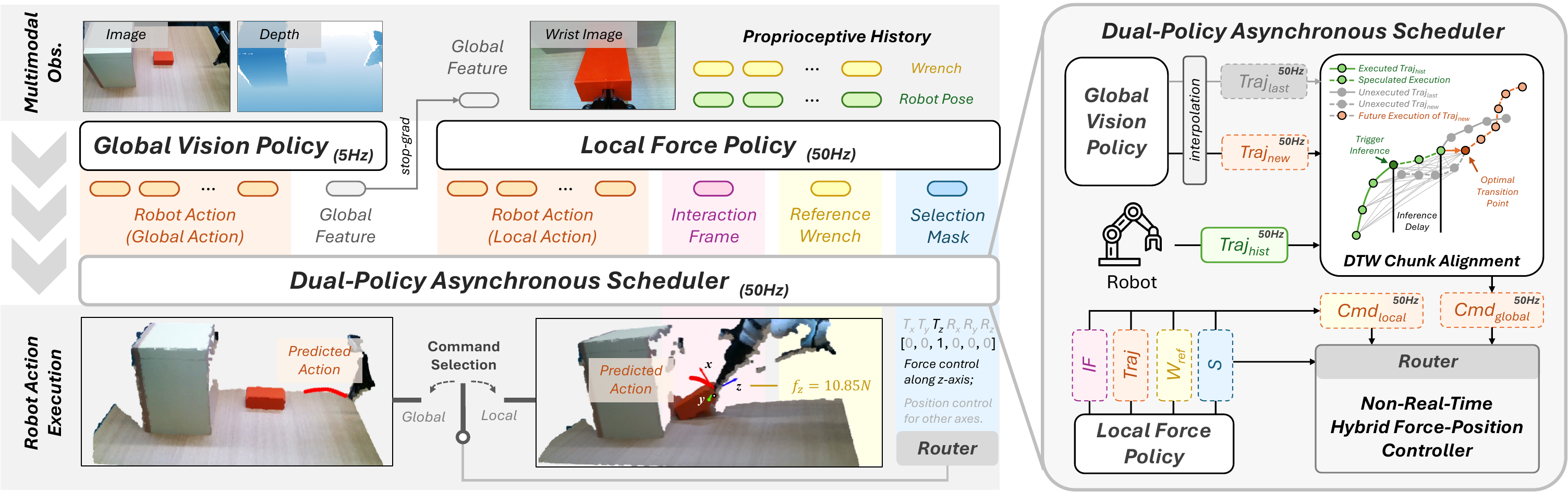}
    \caption{\textbf{\textit{Force Policy} and Dual-Policy Asynchronous Scheduler.} \textit{(Left)} \textbf{\textit{Force Policy}} consists of a global vision policy and a local force policy. The global policy provides task-level visual context and global actions, while the local policy predicts interaction structure and local actions to realize hybrid force-position control during contact. \textit{(Right)} The dual-policy asynchronous scheduler switches between the two policies and reduces latency and jerk via model-agnostic chunk alignment using dynamic time warping (DTW)~\cite{dtw}. }
    \label{fig:policy}\vspace{-0.5cm}
\end{figure*}

\subsection{\textbf{\textit{Force Policy}}: A Global-Local Vision-Force Policy}\label{sec:forcepolicy}

\textbf{Global-Local Vision-Force Decomposition.}
Contact-rich manipulation alternates between two regimes: vision-guided motion in free space and force-governed interaction at contact. Hence, we factor the policy accordingly. 
\emph{Global motion} is handled by a vision policy that reasons over geometry and long-horizon sequencing; outside contact, force readings are largely uninformative and often dominated by sensor noise or incidental touches~\cite{foar}. 
\emph{Local interaction} is handled by a high-frequency force policy that reacts to end-effector proprioception and force/torque signals to regulate contact dynamics. 
\textit{This decomposition assigns each modality to the scale where it is most informative: vision drives task progress, while force enables responsive, stable interaction control}.

\textbf{Modeling Local Force Policy.}
Human interaction with the physical world offers a useful intuition for designing the local interaction policy. Consider the simple act of sliding a finger along the edge of a table. Once contact is established, precise visual perception of the geometry is unnecessary; vision mainly provides \textit{global contextual cues} about where interaction occurs. The interaction itself is regulated through high-frequency proprioception and force feedback: motion sensed by the hand informs how movement evolves, while force feedback reveals contact, resistance, and slip. Therefore, we formulate the fast, local, force policy $\Pi_\text{local}$ as:
\begin{equation}
\mathbf{a}_t^\text{local} = \Pi_\text{local} \Big( \Delta s_{t-T_o+1:t}, \;\boldsymbol{\mathcal{W}}_{t-T_o+1:t}, \;\phi(I_{t'}) \Big),
\end{equation}
where $\mathbf{a}_t^\text{local}$ is the predicted local action, $\Delta s_{t-T_o+1:t}$ and $\boldsymbol{\mathcal{W}}_{t-T_o+1:t}$ denotes the end-effector motion and the force feedback histories of length $T_o$ respectively. $\phi(I_{t'})$ is the global scene context updated by the global vision policy at $t' \le t$, decoupling the local and global policy inference frequencies.

Each predicted local action explicitly parameterizes force-based interaction: it specifies how motion and force should be regulated in IF, rather than directly commanding joint torques or poses. Formally, we define
\begin{equation}
\mathbf{a}_t^\text{local} \triangleq \Big( \Sigma_{t+1},\; S_{t+1},\; \boldsymbol{\mathcal{W}}^\text{ref}_{t+1},\; \Delta s_{t+1:t+T_a}\Big),
\end{equation}
where $S_{t+1}$ selects the controlled motion and force subspaces in IF $\Sigma_{t+1}$, $\Delta s_{t+1:t+T_a}$ specifies the desired motion chunk~\cite{act} of horizon $T_a$, and $\boldsymbol{\mathcal{W}}^\text{ref}_{t+1}$ provides the reference wrench that directly governs force regulation during contact. 

\textbf{Modeling Global Vision Policy.} 
The global vision policy $\Pi_{\text{global}}$ governs task-level progression. In free space, it acts as a standard vision-based policy and directly outputs actions. During interaction, control is delegated to the local force policy $\Pi_\text{local}$, and the global vision policy no longer issues commands. Instead, it provides a global visual feature $\phi(I_{t'})$ to condition local interaction control. This modular design allows $\Pi_{\text{global}}$ to be instantiated from existing visuomotor policies~\cite{act,dp,rise,rise2} or vision-language-action (VLA) models~\cite{pi0,pi05, gr00t} without modifications.

\textbf{Router.} \label{sec:router}
The router determines which policy holds control authority between the global vision policy and the local force policy. Rather than introducing an explicit routing module, we \textit{reuse the predicted selection mask $S_{t+1}$ from the local force policy as an implicit routing signal}. The local force policy dynamically activates force-control axes as needed by assigning non-zero entries in $S_{t+1}$, thereby switching control from the global policy to engage hybrid force-position regulation. In contrast, unintended contacts do not trigger the selection mask, and control remains with the global vision policy.

\textbf{\textit{Force Policy}.}
As shown in Fig.~\ref{fig:policy} (left), the overall policy is decomposed into a global vision policy $\Pi_\text{global}$ and a local force policy $\Pi_\text{local}$. $\Pi_\text{global}$ is responsible for global perception and high-level action generation, while $\Pi_\text{local}$ handles local interaction through force-aware feedback at a high frequency. In our implementation, wrist-mounted camera observations are fed into $\Pi_\text{local}$, as they lie within the local interaction field. We instantiate $\Pi_\text{global}$ with RISE-2~\cite{rise2}, a generalizable visuomotor policy that relies solely on global 3D visual perception. In its pipeline, we treat the action feature used as the conditions in the diffusion head as the global visual feature $\phi(I_{t'})$.
For $\Pi_\text{local}$, the wrist image and proprioceptive history are encoded via a lightweight ResNet~\cite{resnet} and a GRU~\cite{gru}, respectively. Both are conditioned on $\phi(I_t)$ via FiLM~\cite{film}. The resulting features are adaptively fused through a gated mechanism and used to denoise the action sequence with an MIP head~\cite{mip}, as well as to predict the interaction frame, reference wrench, and selection mask via an MLP head.

\subsection{Dual-Policy Asynchronous Scheduler}\label{sec:scheduler}

We introduce a dual-policy asynchronous scheduler during deployment to coordinate a global vision policy with a high-frequency local force policy, while preventing global inference latency from degrading closed-loop control. The scheduler (1) runs both policies asynchronously and uses a router with smooth state transitions to avoid discontinuities during policy switching; and (2) compensates for delayed global outputs by time-aligning and merging them into the execution stream in a latency-aware manner, enforcing trajectory smoothness so executed motions remain stable and physically consistent. This smoothing also reduces jerk-induced inertial force transients, improving force-sensing fidelity during manipulation. An overview is shown in Fig.~\ref{fig:policy} (right).

\begin{figure*}
    \centering
    \includegraphics[width=\linewidth]{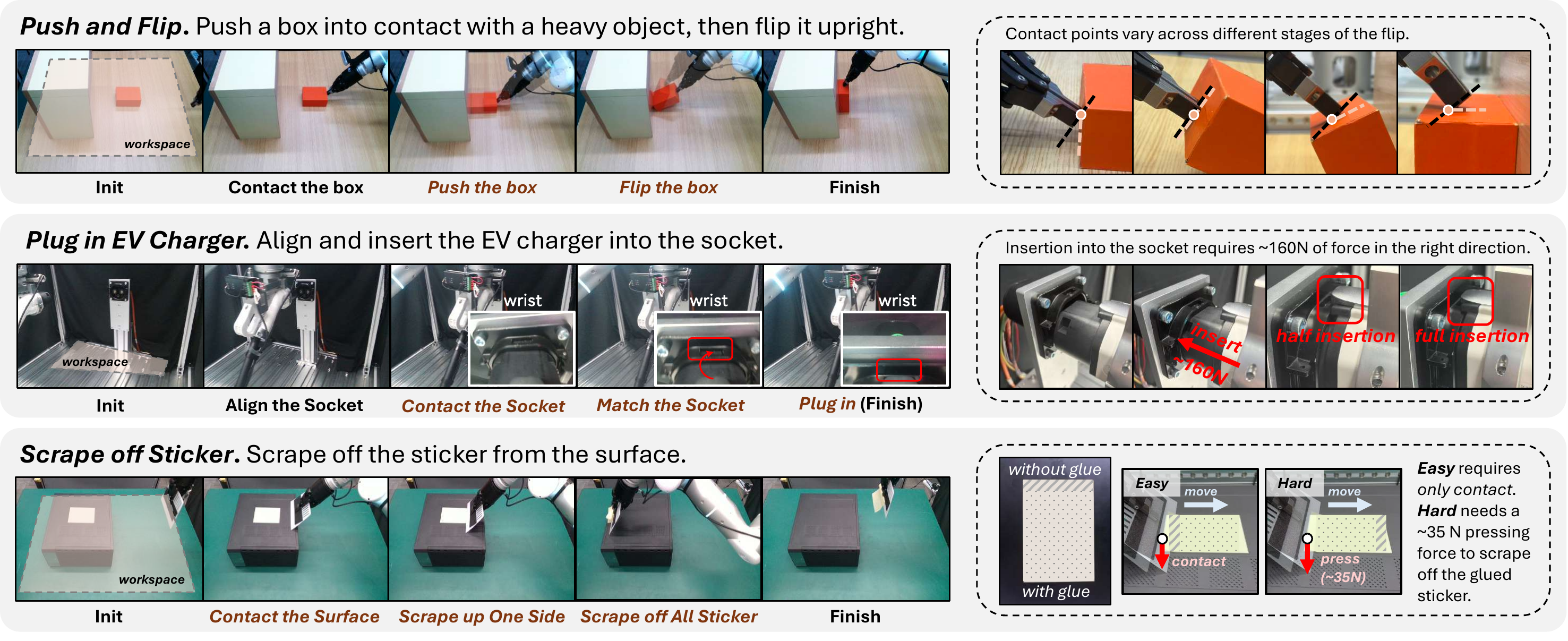}
    \caption{\textbf{Tasks.} We design three tasks spanning two categories (polishing and insertion) to evaluate different policies for contact-rich manipulation. The descriptions on the right highlight the key challenges of each task compared to similar tasks in prior literature. All tasks require highly accurate force regulation to be successfully completed. We randomize object placement within the workspace area during both data collection and evaluation for each task.}
    \label{fig:tasks}\vspace{-0.5cm}
\end{figure*}

\textbf{Asynchronous Inference and Execution.}
To mask inference latency, the scheduler adopts a multi-threaded asynchronous design that overlaps model inference with trajectory execution, preemptively launching inference at fixed intervals. Action selection follows the router in \S\ref{sec:router}, with hysteresis applied to suppress switching jitter. In implementation, outputs from both the global vision policy and the local force policy are first interpolated to a unified frequency of 50Hz. At this common rate, waypoints are selectively dropped to mitigate chunk mismatch induced by inference latency, as detailed in the following paragraph, followed by acceleration-continuous trajectory planning. This planning step ensures smooth execution and high-fidelity force control, as higher-order continuity suppresses jerk and reduces inertial disturbances. The resulting trajectories are then dispatched to a non-real-time hybrid force-position controller~\cite{tang2024partially} for execution.

\textbf{Chunk Alignment with Dynamic Time Warping.}
Latency in asynchronous inference systems typically causes discontinuous jumps between consecutive action chunks. Existing methods~\cite{sail, rtc, training_rtc} address this issue via action inpainting or adaptive conditional guidance, but depend on specific model architectures and introduce additional computational overhead. In contrast, we propose a computationally efficient, model-agnostic waypoint dropout strategy based on Dynamic Time Warping (DTW)~\cite{dtw}. DTW enables nonlinear temporal alignment and has been widely adopted from speech recognition~\cite{1163055, 1163491, issam2025dtwalignbridgingmodalitygap} to motion retargeting~\cite{liu2025immimic}. By aligning the predicted trajectory with the recent execution history, we can identify an optimal entry index and drop preceding waypoints, thereby eliminating latency-induced discontinuities without modifying the training procedure or model architecture. Additional details are provided in the supplementary. 

The proposed scheduler enables latency-aware deployment of dual-policy control by decoupling asynchronous inference from execution while maintaining smooth, force-consistent motion. Unified-frequency resampling, DTW-based chunk alignment, and acceleration-continuous planning together ensure stable execution without constraining policy architecture.
\section{Experiments}

Through experiments, we seek to answer the following questions: \textbf{(Q1)} Can \textbf{\textit{Force Policy}} effectively handle different types of contact-rich tasks? \textbf{(Q2)} Does the \textbf{\textit{Force Policy}} achieve superior force control compared to existing force-aware baselines? \textbf{(Q3)} Does our global-local design generalize better than the baselines? \textbf{(Q4)} Does our IF recovery method produce more accurate IF than prior approaches, and does improved interaction labeling enable the policy to learn more accurate contact behaviors? \textbf{(Q5)} Does the asynchronous scheduler reduce motion and force jerks, leading to smoother trajectories during deployment?

\begin{table*}
    \centering
    \caption{\textbf{Evaluation Success Rates.} We report success rates evaluated in an accumulative, stage-wise manner. \textbf{\textit{Force Policy}} achieves the highest success rates compared to vision-based and force-aware baselines among all tasks. }
    \label{tab:result}
    \begin{tabular}{crrrrrrrrrrr}
    \toprule
         \multirow{2}{*}{\textbf{Policy}} & \multicolumn{2}{c}{\textbf{\textit{Push and Flip}}} & \multicolumn{3}{c}{\textbf{\textit{Plug in EV Charger}}} & \multicolumn{3}{c}{\textbf{\textit{Scrape off Sticker}} (Easy)} & \multicolumn{3}{c}{\textbf{\textit{Scrape off Sticker}} (Hard)} \\ \cmidrule(lr){2-3}\cmidrule(lr){4-6}\cmidrule(lr){7-9}\cmidrule(lr){10-12}
         & \multicolumn{1}{c}{push} & \multicolumn{1}{c}{flip} & \multicolumn{1}{c}{contact} & \multicolumn{1}{c}{match} & \multicolumn{1}{c}{plug in} & \multicolumn{1}{c}{contact} & \multicolumn{1}{c}{one-side} & \multicolumn{1}{c}{full off} & \multicolumn{1}{c}{contact} & \multicolumn{1}{c}{one-side} & \multicolumn{1}{c}{full off}\\
         \midrule
         RISE-2~\cite{rise2} & \textbf{100.0}\% & 42.5\% & \textbf{100.0}\% & \textbf{90.0}\% & 0.0\% & \textbf{100.0}\% & 80.0\% & 80.0\% & \textbf{100.0}\% & 20.0\% & 10.0\%\\
         $\pi_{0.5}$~\cite{pi05} & 80.0\% & 52.5\% & \textbf{100.0}\% & 30.0\% & 0.0\% & 70.0\% & 70.0\% & 70.0\% & \textbf{100.0}\% & 40.0\% & 20.0\% \\ \midrule
         RDP~\cite{rdp} & \textbf{100.0}\% & 57.5\% & \textbf{100.0}\% & 70.0\% & 5.0\% & \textbf{100.0}\% & 90.0\% & 70.0\% & \textbf{100.0}\% & 30.0\% & 20.0\% \\ 
         FoAR~\cite{foar} & \textbf{100.0}\% & 60.0\% & \textbf{100.0}\% & 80.0\% & 10.0\% & \textbf{100.0}\% & 70.0\% & 40.0\% & \textbf{100.0}\% & 60.0\% & 20.0\% \\
         ForceVLA~\cite{forcevla} & 50.0\% & 30.0\% & 20.0\% & 0.0\% & 0.0\% & 10.0\% & 0.0\% & 0.0\% & 10.0\% & 0.0\% & 0.0\%\\
         TA-VLA~\cite{tavla} & 85.0\% & 62.5\% & 80.0\% & 10.0\% & 0.0\% & 50.0\% & 50.0\% & 50.0\% & 40.0\% & 20.0\% & 10.0\%\\
         \midrule
         \cellcolor[HTML]{f2f2f2}\textbf{\textit{Force Policy}} (ours) & \cellcolor[HTML]{f2f2f2}\textbf{100.0}\% & \cellcolor[HTML]{f2f2f2}\textbf{95.0}\% & \cellcolor[HTML]{f2f2f2}\textbf{100.0}\% & \cellcolor[HTML]{f2f2f2}\textbf{90.0}\% & \cellcolor[HTML]{f2f2f2}\textbf{65.0}\% & \cellcolor[HTML]{f2f2f2}\textbf{100.0}\% & \cellcolor[HTML]{f2f2f2}\textbf{100.0}\% & \cellcolor[HTML]{f2f2f2}\textbf{100.0}\% & \cellcolor[HTML]{f2f2f2}\textbf{100.0}\% & \cellcolor[HTML]{f2f2f2}\textbf{90.0}\% & \cellcolor[HTML]{f2f2f2}\textbf{90.0}\% \\ \bottomrule
    \end{tabular}
    \vspace{-0.25cm}
\end{table*}

\subsection{Setup}

\textbf{Platform.} The robot platform comprises a Flexiv Rizon 4 arm equipped with a Flexiv GN-02 gripper and a 6-DoF force-torque sensor at the flange. Two Intel RealSense D415 RGB-D cameras are used to capture visual observations, serving as a global camera and a wrist-mounted camera, respectively.

\textbf{Tasks.} As shown in Fig.~\ref{fig:tasks}, we design three tasks from two major applications: polishing and insertion, including scenarios with continuously varying IFs and tasks that demand reactive control and accurate force regulation.

\textbf{Data Collection.} We use the arm-to-arm teleoperation toolkit with force feedback~\cite{tdk} to collect 50 high-quality demonstrations for each task. The system provides high-fidelity force feedback during teleoperation, enabling precise force regulation across tasks and task phases. Inspection of all demonstrations shows that the recorded force signals consistently fall within the required range for each task stage.

\textbf{Baselines.} For vision-only baselines, we consider two state-of-the-art policies: RISE-2~\cite{rise2} and $\pi_{0.5}$~\cite{pi05}. For force-aware baselines, we include the representative RDP~\cite{rdp}, FoAR~\cite{foar}, ForceVLA~\cite{forcevla}, and TA-VLA~\cite{tavla} in evaluations.

\textbf{Metrics.} We report success rate as the primary metric for each task. For the ``flip'' phase in \textbf{\textit{Push and Flip}} and the ``plug in'' phase in \textbf{\textit{Plug in EV Charger}}, we also count partial completion as a half success (0.5) per trial.

\textbf{Evaluation.} The $\pi$-based policies run on a server with an NVIDIA A800 GPU due to memory limitations, while others run on a workstation with an RTX 3090. Following~\cite{umi, cage}, all policies are evaluated consistently: test positions are randomly generated beforehand, the workspace is identical, and metrics are counted over 20 trials for \textbf{\textit{Push and Flip}}, and 10 for others.

\subsection{Results}

\textbf{\textit{Force Policy} exhibits significant effectiveness in handling a diverse range of contact-rich tasks (Q1).} As shown in Tab.~\ref{tab:result}, \textbf{\textit{Force Policy}} consistently outperforms both vision-only and force-aware baselines across a wide range of contact-rich tasks, with the largest gains appearing in phases that require precise and rapid force regulation, \textit{e.g.}, during the insertion stage in the \textbf{\textit{Plug in EV Charger}} task, and the scraping stage in the \textbf{\textit{Scrape off Sticker}} task. The underlying insight is that \textit{force feedback is genuinely useful, but only when integrated in a way that matches its intermittent, regime-dependent nature}.

\begin{figure}[t]
    \centering
    \includegraphics[width=\linewidth]{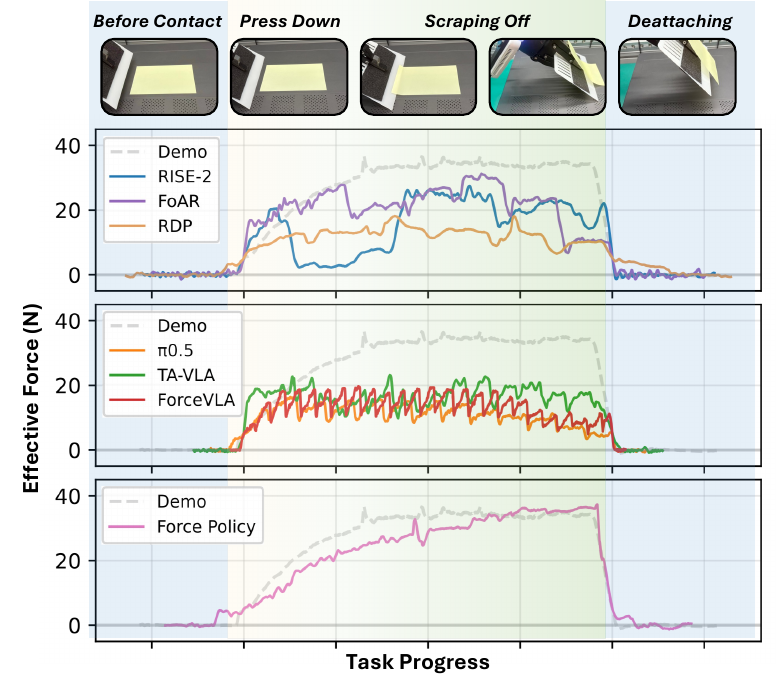}\vspace{-0.1cm}
    \caption{\textbf{Visualization of Effective Forces during Deployment and from Demonstrations on the \textit{Scrape off Sticker} (Hard) Task.} All baselines fail to replicate the force behavior from demonstrations, resulting in degraded performance. \textit{\textbf{Force Policy}} closely imitates the effective force in demonstrations, achieving higher success rates.}
    \label{fig:force_scrape}\vspace{-0.5cm}
\end{figure}

Many prior designs misuse force in ways that reduce robustness. (1) Monolithic vision-force policies that simply concatenate force into observations, such as ForceVLA and TA-VLA, are prone to failure, because \textit{force is noisy or uninformative during non-contact motion} and can pollute global decision-making~\cite{foar}, sometimes even underperforming their vision-only backbone. The effect is particularly pronounced in the evaluation results of ForceVLA. (2) Low-frequency force-aware pipelines such as FoAR \textit{lack sufficient reactivity to handle fast contact transients}. (3) RDP employs a hierarchically-coupled policy design, in which low-level force-aware action decoding is directly conditioned on the high-level latent action chunk. When the high-level policy encodes a no-contact latent action but unexpected contact occurs during execution, the low-level tokenizer may map out-of-distribution contact force signals to erroneous actions. \textit{This coupling limits the ability of low-level policy to correct wrong high-level latent action}.

In contrast to all baselines, our global-local vision-force design decouples high-level intent from low-level execution: the global vision module provides a stable long-horizon reference, \textit{insulated from noisy force signals}, while the high-frequency local force policy \textit{independently determines actions} based on this reference and real-time force feedback, enabling robust performance under sensor noise and fast contact transients.

\begin{figure}[t]
    \centering
    \captionof{table}{\textbf{Force Regulation Evaluation on the \textit{Push and Flip} Task.} \textit{(Left)} Pushing the heavy object requires approximately 45N, while demonstrations apply about 15N pushing force to flip the target object; we then measure its pushed distance $d$ as an indicator of force regulation. \textit{(Right)} Statistics of the pushed distance for each method. }
    \label{tab:distance}
    \begin{minipage}{0.4\linewidth}
    \centering
    \includegraphics[width=\linewidth]{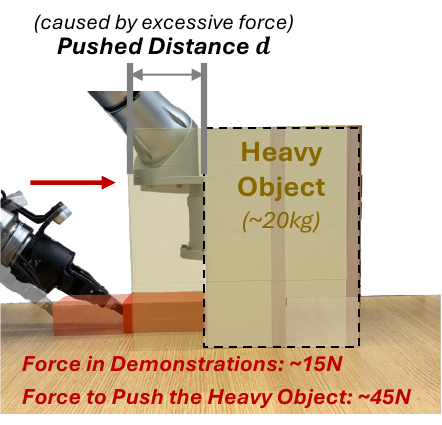}
    \end{minipage}
    \hfill
    \begin{minipage}{0.55\linewidth}
    \centering
    \setlength\tabcolsep{2pt}
    \footnotesize
    \begin{tabular}{cc}
    \toprule
       \textbf{Policy}  &  \textbf{Avg.} $d$ (cm) $\downarrow$ \\ \midrule
       RISE-2~\cite{rise2} & $0.86$\\ 
       $\pi_{0.5}$~\cite{pi05} & $\textbf{0.00}$\\ \midrule
       RDP~\cite{rdp} & $1.13$\\ 
       FoAR~\cite{foar} & $0.25$\\ 
       ForceVLA~\cite{forcevla} & $0.50$\\ 
       TA-VLA~\cite{tavla} & $1.25$\\ \midrule
       \textbf{\textit{Force Policy}} (ours) & $\textbf{0.00}$\\ \bottomrule
    \end{tabular}
    \end{minipage}\vspace{-0.4cm}
\end{figure}

\textbf{\textit{Force Policy} achieves significantly superior force control compared to existing force-aware baselines (Q2).} As visualized in Fig.~\ref{fig:force_scrape} and Tab.~\ref{tab:distance}, our method closely tracks the force profile of human demonstrations, maintaining the necessary effective force magnitude throughout the critical phases. On the contrary, baselines typically exhibit severe force oscillations, fail to exert sufficient downward force, or apply excessive force, resulting in task failure. This performance advantage stems from our adoption of a hybrid force-position control framework. Unlike baselines that regulate force implicitly or treat it as a sensory input, our approach explicitly controls interaction forces, enabling stable execution that faithfully follows demonstrations, which is essential for contact-rich tasks requiring precise force control.

\textbf{\textit{Force Policy} demonstrates decent generalization ability compared to baselines (Q3).} We aim to evaluate two aspects of generalization: visual generalization to \textit{unseen objects}, and force-regulation generalization to \textit{unseen geometries} during contact. The evaluation results of the policies on unseen objects with varying colors, geometries, and stiffnesses in the \textbf{\textit{Push and Flip}} task are demonstrated in Tab.~\ref{tab:generalization}. Force Policy consistently achieves high success rates, whereas baselines frequently suffer catastrophic failure. Crucially, we must clarify that \textit{although RISE-2, our global vision policy, often yields a 0\% success rate on novel objects, it successfully navigates to the near-contact region in most trials}. Its failure mainly arises from relying on vision alone to handle the critical transition from free motion to contact, causing timing errors or improper contact establishment. \textit{Our design successfully inherits this visual generalization capability from the global vision backbone to reliably guide the end-effector to the target vicinity}. Once in the near-contact region, the local force policy takes over, leveraging force feedback and hybrid force-position control to robustly establish contact and execute the actions. By decoupling visual reaching from physical interaction, our method effectively combines the broad generalization of visuomotor policies with the precision of local force control.

\begin{table}[t]
\vspace{0.15cm}
\centering
\caption{\textbf{Generalization Evaluation on the \textit{Push and Flip} Task.} We use unseen objects of different colors, geometries, and stiffnesses to evaluate the generalization ability of force-aware policies. Please refer to the supplementary material for detailed object comparisons.}
\label{tab:generalization}
\setlength\tabcolsep{2pt}
\begin{tabular}{ccccccc}
\toprule
\diagbox[width=2cm,height=0.9cm]{\textbf{Policy}}{\textbf{Unseen Obj.}}
& \raisebox{-0.5\height}{\includegraphics[width=0.9cm]{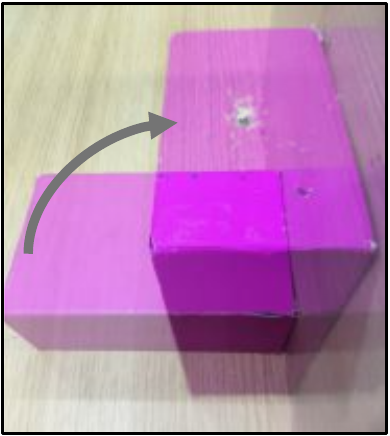}}
& \raisebox{-0.5\height}{\includegraphics[width=0.9cm]{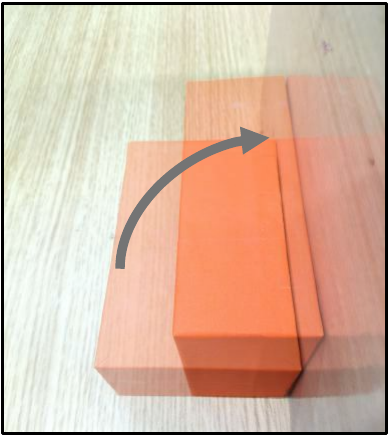}}
& \raisebox{-0.5\height}{\includegraphics[width=0.9cm]{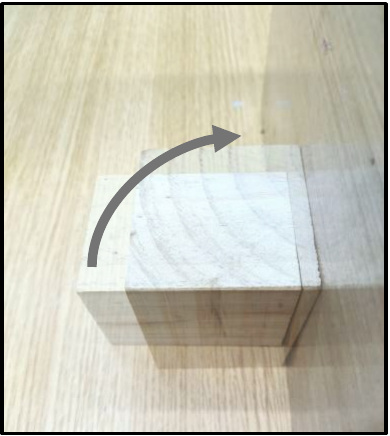}}
& \raisebox{-0.5\height}{\includegraphics[width=0.9cm]{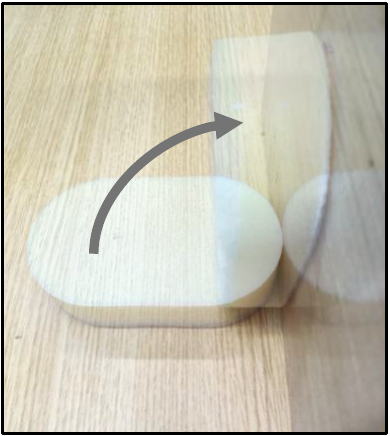}}
& \raisebox{-0.5\height}{\includegraphics[width=0.9cm]{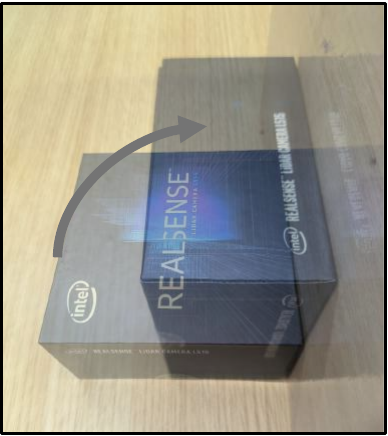}}
& \raisebox{-0.5\height}{\includegraphics[width=0.9cm]{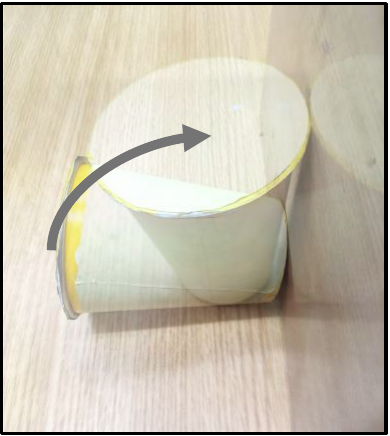}} \\
\midrule
 RISE-2~\cite{rise2} & 3/5 & 0/5 & 0/5 & 3/5 & 0/5 & 0/5\\
$\pi_{0.5}$~\cite{pi05} & 2/5 & 2/5 & 1/5 & 3/5 & \textbf{3/5} & 1/5 \\ \midrule
RDP~\cite{rdp} &  2/5 & 0/5 & 0/5 & 0/5 & 0/5 & 0/5\\
FoAR~\cite{foar} & 1/5 & 0/5 & 0/5 & 0/5 & 0/5 & 0/5\\
ForceVLA~\cite{forcevla} & 2/5 & 0/5 & 0/5 & 1/5 & 0/5 & 0/5\\
TA-VLA~\cite{tavla} & 3/5 & 2/5 & 1/5 & 1/5 & \textbf{3/5} & 1/5\\ \midrule
\cellcolor[HTML]{f2f2f2} \textbf{\textit{Force Policy}} (ours) & \cellcolor[HTML]{f2f2f2}\textbf{5/5} & \cellcolor[HTML]{f2f2f2}\textbf{5/5} & \cellcolor[HTML]{f2f2f2}\textbf{4/5} & \cellcolor[HTML]{f2f2f2}\textbf{4/5} & \cellcolor[HTML]{f2f2f2}\textbf{3/5} & \cellcolor[HTML]{f2f2f2}\textbf{2/5} \\
\bottomrule
\end{tabular}\vspace{-0.4cm}
\end{table}

\subsection{Ablations}

\begin{figure}[b]
    \centering
    \includegraphics[width=\linewidth]{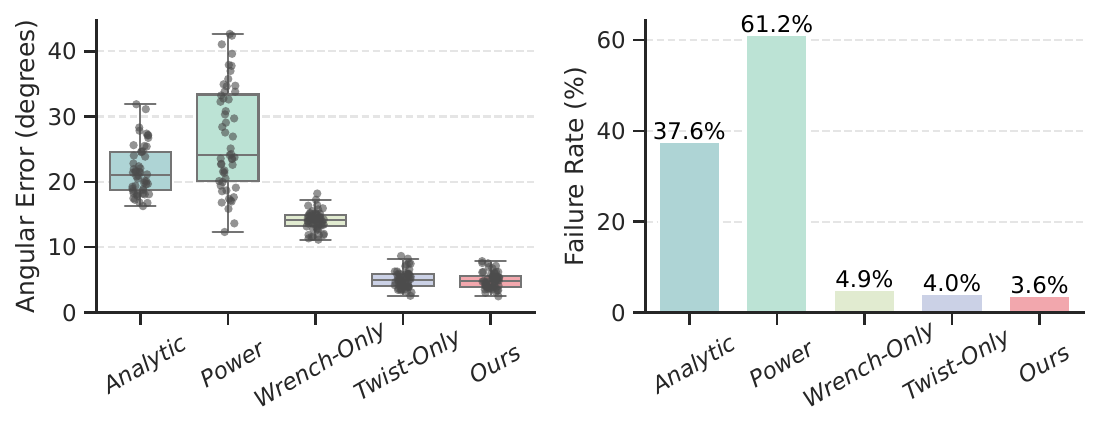}
    \caption{\textbf{IF Recovery Evaluation on the \textit{Scrape off Sticker} Task.} Angular error is computed between the recovered and ground-truth force control axes; a failure is counted if it exceeds 20$^\circ$.}
    \label{fig:if_evaluation}
\end{figure}

\textbf{Our IF recovery method achieves more accurate interaction frames from demonstrations than previous approaches, improving interaction quality (Q4).} We evaluate several IF recovery approaches on the \textit{Scrape off Sticker} task, including analytic~\cite{analytic_if}, power-based~\cite{interaction_frame}, wrench-only~\cite{forcemimic,conkey2019learning}, twist-only, and our adaptive method. The $z$-axis of the ground-truth IF aligns with the world vertical axis in this task, so we measure angular error and count failures above 20$^\circ$. The results are shown in Fig.~\ref{fig:if_evaluation}. Analytic and power-based approaches exhibit larger errors, while twist-only and our method outperform wrench-only due to dissipative residuals like frictions dominating power in this task. Our adaptive method slightly improves over twist-only by capturing brief pressing periods before scraping, where structural residuals like scraper deformation dominate. We further evaluate the \textbf{\textit{Force Policy}} trained with interaction frames labeled by wrench-only methods~\cite{forcemimic} on this task. It achieves only a 50\% success rate, dropping from 90\% with our IF recovery strategy.

\textbf{The semantic classification in our IF recovery method is robust in classifying different contact behaviors in the real world (Q4).} The semantic classification accuracies for the three tasks are 100.0\%, 92.0\%, and 98.0\%, respectively. We argue that \textbf{\textit{Force Policy}} is not sensitive to these errors for the following two reasons: (1) Misclassified samples are infrequent, so their effect on policy learning is limited. (2) Even if an incorrect semantic prior produces a suboptimal IF and the policy learns to predict the suboptimal IF, the high-frequency local policy can still detect abnormal force signals and correct the IF online using the feedback.

\textbf{The asynchronous scheduler substantially reduces jerks and yields smoother trajectories during deployment (Q5).} We analyze both force and motion jerks and measure trajectory smoothness using spectral arc length (SPARC, $\uparrow$)~\cite{sparc} following~\cite{sail}. As shown in Fig.~\ref{fig:scheduler}, the temporal profiles indicate that the scheduler effectively attenuates signal discontinuities, resulting in smoother trajectories during the inference chunk alignment phase. Additional analysis of the asynchronous scheduler is provided in the supplementary material.

\begin{figure}[h]
\vspace{-0.1cm}
    \centering
    \includegraphics[width=\linewidth]{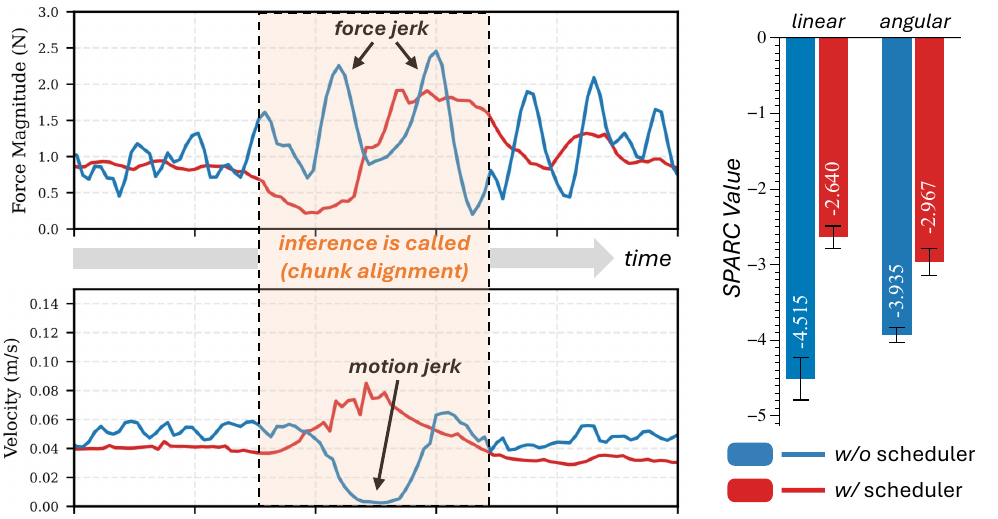}
    \caption{\textbf{Asynchronous Scheduler Evaluation on the \textit{Push and Flip} Task.} The scheduler effectively reduces both motion and force jerks \textit{(left)}, resulting in more smoother executed trajectories \textit{(right)}.}
    \label{fig:scheduler}\vspace{-0.4cm}
\end{figure}

\section{Limitations and Future Works}

\textbf{IF Scope.} We clarify that the present scope of our IF theory encompasses \textit{stable-contact manipulation}, including \textit{most frictional and compliant contact settings}. In these prevalent regimes, explicitly modeling structural residuals and continuous dissipative forces offers a distinct practical advantage over prior methods. Our formulation is grounded in \textit{locally conservative} and \textit{topologically invariant} assumptions. Consequently, highly non-smooth scenarios, such as fracture or non-conservative contact transitions, fall outside this scope and are left for future work. Multi-point and surface contact cases are discussed in the remark to Theorem~\ref{thm:diag} in the supplementary. 

\textbf{Evaluation Scope}. Our current evaluation covers two broad classes of contact-rich tasks, \textit{peg-in-hole} and \textit{surface interaction}, spanning common contact-rich manipulation scenarios. In particular, \textbf{\textit{Push and Flip}} involves a \textit{continuously changing IF}, serving as a dynamic test of both IF recovery and IF prediction. Broader evaluations of complex contact-rich tasks is indeed valuable, and we leave it to future work.

\textbf{Torque Control.} Our current formulation focuses more on recovering interaction orientation for force control, rather than explicitly estimating the true contact point; fully modeling the contact structure would be necessary to support torque control.

\textbf{Contact for High-Level Decision-Making.} For some tasks, force/torque signals might help high-level decision-making, \textit{e.g.}, briefly tugging to check if the connector is fully inserted; a promising direction is to extend the architecture to enable such force-driven judgments.

\section{Conclusion}

Motivated by the biological division of labor where vision guides global decision-making and reaching while haptics governs local contact, we propose \textbf{\textit{Force Policy}}, a global-local vision-force framework that decouples the control hierarchy. The global vision policy decides ``\textit{where to act}'' with strong generalization, while a high-frequency local controller decides ``\textit{how to act}'' by regulating contact via hybrid force-position control. This design is enabled by a physically grounded interaction-frame formulation and a practical method to recover intrinsic frames from demonstrations. Empirically, we show (1) substantially improved robustness and precise force regulation on diverse contact-rich manipulation tasks by explicitly modeling contact structure and interaction forces, and (2) strong generalization across various object geometries and physical properties, driven by the complementary roles of the global vision policy that handles visual and geometric variation and the local force policy that handles contact dynamics. Overall, \textbf{\textit{Force Policy}} bridges the generalization of visuomotor policies with the precision of classical force control, providing a scalable approach to contact-rich manipulation.

\section*{Acknowledgement}

This work is supported by the New Generation Artificial Intelligence --- National Science and Technology Major Project (No. 2025ZD0122901), the Science and Technology Major Project of Jiangsu Province (No. BG2024041), Shanghai Artificial Intelligence Laboratory, and XPLORER PRIZE grants.

We would like to thank Yiming Wang from Shanghai Jiao Tong University for insightful discussions on theoretical formulation; Zhipeng Zhang from Flexiv for his support on the arm-to-arm teleoperation system with force feedback; Wenbo Tang from Flexiv for discussions about the robot controller; Peishen Yan and Chunyu Xue from Shanghai Jiao Tong University for writing advice and proofreading; Lijia Yao from Noematrix for suggestions on figure design; Junchao Zhang from Noematrix for his help on evaluations; and Shangning Xia and Linhao Chen from Noematrix for discussions.

\textit{\textbf{Author Contributions.} H. Fang led the project under the guidance of S. Wang and C. Lu. H. Fang, S. Tang, and M. Mei formulated the core idea of the project together. H. Fang, S. Tang, M. Mei, and S. Wang discussed and derived the theoretical framework of the interaction frame. H. Fang, S. Tang, and M. Mei designed the framework of Force Policy together. H. Fang and S. Tang implement the Force Policy and the asynchronous scheduler, respectively. C. Wang helped with the policy design and scheduler implementation. H. Fang, S. Tang and M. Mei designed the tasks together. M. Mei, Z. He, and H. Qin implemented the baselines. M. Mei, S. Tang, H. Qin, Z. He, J. Chen, Y. Feng, and H. Fang conducted the experiments. H. Fang, S. Tang, M. Mei, J. Chen wrote the paper. S. Wang, C. Lu, W. Liu and Z. He supervised the project and provided hardware support. This work was done while H. Fang was a research collaborator with Noematrix, and M. Mei was a research intern at Noematrix.}

\printbibliography

@article{equicontact,
  title={EquiContact: A Hierarchical SE(3) Vision-To-Force Equivariant Policy for Spatially Generalizable Contact-Rich Tasks},
  author={Seo, Joohwan and Kruthiventy, Arvind and Lee, Soomi and Teng, Megan and Zhang, Xiang and Choi, Seoyeon and Choi, Jongeun and Horowitz, Roberto},
  journal={arXiv preprint arXiv:2507.10961},
  year={2025}
}

@inproceedings{rdp,
  title     = {Reactive Diffusion Policy: Slow-Fast Visual-Tactile Policy Learning for Contact-Rich Manipulation},
  author    = {Xue, Han and Ren, Jieji and Chen, Wendi and Zhang, Gu and Fang, Yuan and Gu, Guoying and Xu, Huazhe and Lu, Cewu},
  booktitle = {Robotics: Science and Systems},
  year      = {2025}
}

@inproceedings{oxe,
  title={Open X-Embodiment: Robotic Learning Datasets and RT-X Models},
  author={Open X-Embodiment Collaboration and O’Neill, Abby and Rehman, Abdul and Maddukuri, Abhiram and Gupta, Abhishek and Padalkar, Abhishek and Lee, Abraham and Pooley, Acorn and Gupta, Agrim and Mandlekar, Ajay and Jain, Ajinkya and others},
  booktitle={IEEE International Conference on Robotics and Automation},
  pages={6892--6903},
  year={2024}
}

@inproceedings{rh20t,
  author       = {Hao-Shu Fang and
                  Hongjie Fang and
                  Zhenyu Tang and
                  Jirong Liu and
                  Chenxi Wang and
                  Junbo Wang and
                  Haoyi Zhu and
                  Cewu Lu},
  title        = {{RH20T:} {A} Comprehensive Robotic Dataset for Learning Diverse Skills in One-Shot},
  booktitle    = {IEEE International Conference on Robotics and Automation},
  pages        = {653--660},
  year         = {2024},
}

@inproceedings{rt1,
  author       = {Anthony Brohan and
                  Noah Brown and
                  Justice Carbajal and
                  Yevgen Chebotar and
                  Joseph Dabis and
                  Chelsea Finn and
                  Keerthana Gopalakrishnan and
                  Karol Hausman and
                  Alexander Herzog and
                  Jasmine Hsu and
                  Julian Ibarz and
                  Brian Ichter and
                  Alex Irpan and
                  Tomas Jackson and
                  Sally Jesmonth and
                  Nikhil J. Joshi and
                  Ryan Julian and
                  Dmitry Kalashnikov and
                  Yuheng Kuang and
                  Isabel Leal and
                  Kuang{-}Huei Lee and
                  Sergey Levine and
                  Yao Lu and
                  Utsav Malla and
                  Deeksha Manjunath and
                  Igor Mordatch and
                  Ofir Nachum and
                  Carolina Parada and
                  Jodilyn Peralta and
                  Emily Perez and
                  Karl Pertsch and
                  Jornell Quiambao and
                  Kanishka Rao and
                  Michael S. Ryoo and
                  Grecia Salazar and
                  Pannag R. Sanketi and
                  Kevin Sayed and
                  Jaspiar Singh and
                  Sumedh Sontakke and
                  Austin Stone and
                  Clayton Tan and
                  Huong T. Tran and
                  Vincent Vanhoucke and
                  Steve Vega and
                  Quan Vuong and
                  Fei Xia and
                  Ted Xiao and
                  Peng Xu and
                  Sichun Xu and
                  Tianhe Yu and
                  Brianna Zitkovich},
  title        = {{RT-1:} Robotics Transformer for Real-World Control at Scale},
  booktitle    = {Robotics: Science and Systems},
  year         = {2023}
}

@inproceedings{pi0,
  title={$\pi_0$: A Vision-Language-Action Flow Model for General Robot Control},
  author={Black, Kevin and Brown, Noah and Driess, Danny and Esmail, Adnan and Equi, Michael and Finn, Chelsea and Fusai, Niccolo and Groom, Lachy and Hausman, Karol and Ichter, Brian and others},
  booktitle={Robotics: Science and Systems},
  year={2025}
}

@inproceedings{rot6d,
  title={On the Continuity of Rotation Representations in Neural Networks},
  author={Zhou, Yi and Barnes, Connelly and Lu, Jingwan and Yang, Jimei and Li, Hao},
  booktitle={IEEE/CVF Conference on Computer Vision and Pattern Recognition},
  pages={5745--5753},
  year={2019}
}

@article{moe,
  title={Adaptive mixtures of local experts},
  author={Jacobs, Robert A and Jordan, Michael I and Nowlan, Steven J and Hinton, Geoffrey E},
  journal={Neural computation},
  volume={3},
  number={1},
  pages={79--87},
  year={1991},
  publisher={MIT Press}
}

@inproceedings{rise,
    title     = {RISE: 3D Perception Makes Real-World Robot Imitation Simple and Effective},
    author    = {Wang, Chenxi and Fang, Hongjie and Fang, Hao-Shu and Lu, Cewu},
    booktitle = {IEEE/RSJ International Conference on Intelligent Robots and Systems}, 
    year      = {2024},
    pages     = {2870-2877}
}

@inproceedings{dp,
	title={Diffusion Policy: Visuomotor Policy Learning via Action Diffusion},
	author={Chi, Cheng and Feng, Siyuan and Du, Yilun and Xu, Zhenjia and Cousineau, Eric and Burchfiel, Benjamin and Song, Shuran},
	booktitle={Robotics: Science and Systems},
	year={2023}
}

@article{foar,
  title={FoAR: Force-Aware Reactive Policy for Contact-Rich Robotic Manipulation},
  author={He, Zihao and Fang, Hongjie and Chen, Jingjing and Fang, Hao-Shu and Lu, Cewu},
  journal={IEEE Robotics and Automation Letters},
  year={2025},
  publisher={IEEE}
}

@article{interaction_frame,
  title={Identifying Physical Interactions in Contact-Based Robot Manipulation for Learning from Demonstration},
  author={Overbeek, Alex Harm Gert-Jan and van der Kooij, Herman and Dresscher, Douwe and Vlutters, Mark},
  journal={Advanced Robotics Research},
  pages={e202500109},
  year={2025},
  publisher={Wiley Online Library}
}

@article{suomalainen2022survey,
  title={A Survey of Robot Manipulation in Contact},
  author={Suomalainen, Markku and Karayiannidis, Yiannis and Kyrki, Ville},
  journal={Robotics and Autonomous Systems},
  volume={156},
  pages={104224},
  year={2022},
  publisher={Elsevier}
}

@inproceedings{acp,
  title={Adaptive Compliance Policy: Learning Approximate Compliance for Diffusion Guided Control}, 
  author={Hou, Yifan and Liu, Zeyi and Chi, Cheng and Cousineau, Eric and Kuppuswamy, Naveen and Feng, Siyuan and Burchfiel, Benjamin and Song, Shuran},
  booktitle={IEEE International Conference on Robotics and Automation}, 
  pages={4829-4836},
  year={2025}
}

@article{zhou2025admittance,
  title={Admittance Visuomotor Policy Learning for General-Purpose Contact-Rich Manipulations},
  author={Zhou, Bo and Jiao, Ruixuan and Li, Yi and Yuan, Xiaogang and Fang, Fang and Li, Shihua},
  journal={IEEE Transactions on Industrial Electronics},
  year={2025},
  publisher={IEEE}
}

@article{forge,
  title={FORGE: Force-Guided Exploration for Robust Contact-Rich Manipulation Under Uncertainty},
  author={Noseworthy, Michael and Tang, Bingjie and Wen, Bowen and Handa, Ankur and Kessens, Chad and Roy, Nicholas and Fox, Dieter and Ramos, Fabio and Narang, Yashraj and Akinola, Iretiayo},
  journal={IEEE Robotics and Automation Letters},
  year={2025},
  publisher={IEEE}
}

@inproceedings{kalakrishnan2011learning,
  title={Learning Force Control Policies for Compliant Manipulation},
  author={Kalakrishnan, Mrinal and Righetti, Ludovic and Pastor, Peter and Schaal, Stefan},
  booktitle={IEEE/RSJ International Conference on Intelligent Robots and Systems},
  pages={4639--4644},
  year={2011}
}

@article{analytic_if,
  author       = {Ali Mousavi Mohammadi and
                  Maxim Vochten and
                  Erwin Aertbeli{\"{e}}n and
                  Joris De Schutter},
  title        = {A Generic Task Model and Control Strategy to Support Learning, Robust Control, and Generalization of Contact-Rich Manipulation Tasks},
  journal      = {Robotics and Autonomous Systems},
  volume       = {197},
  pages        = {105270},
  year         = {2026}
}

@article{kinematic_only,
  author       = {Alex H. G. Overbeek and
                  Douwe Dresscher and
                  Herman van der Kooij and
                  Mark Vlutters},
  title        = {Versatile Kinematics-Based Constraint Identification Applied to Robot Task Reproduction},
  journal      = {Frontiers in Robotics and AI},
  volume       = {12},
  year         = {2025}
}

@inproceedings{rise2,
  title   = {AirExo-2: Scaling up Generalizable Robotic Imitation Learning with Low-Cost Exoskeletons},
  author  = {Hongjie Fang and Chenxi Wang and Yiming Wang and Jingjing Chen and Shangning Xia and Jun Lv and Zihao He and Xiyan Yi and Yunhan Guo and Xinyu Zhan and Lixin Yang and Weiming Wang and Cewu Lu and Hao-Shu Fang},
  booktitle = {Conference on Robot Learning},
  year    = {2025},
  pages = {198--220},
  volume = 	 {305},
  publisher =    {PMLR},
}

@inproceedings{act,
	title={Learning Fine-Grained Bimanual Manipulation with Low-Cost Hardware},
	author={Zhao, Tony Z and Kumar, Vikash and Levine, Sergey and Finn, Chelsea},
	booktitle={Robotics: Science and Systems},
	year={2023}
}

@article{tang2024partially,
  title={Partially Decoupled Impedance Motion Force Control Using Prioritized Inertia Shaping},
  author={Tang, Wenbo and Wang, Weiming and Wang, Shiquan},
  journal={IEEE Robotics and Automation Letters},
  year={2024},
  publisher={IEEE}
}

@book{screw_theory,
  title={A Treatise on the Theory of Screws},
  author={Ball, Robert Stawell},
  year={1998},
  publisher={Cambridge university press}
}

@article{mason_hybrid,
  title={Compliance and Force Control for Computer Controlled Manipulators},
  author={Mason, Matthew T},
  journal={IEEE Transactions on Systems, Man, and Cybernetics},
  volume={11},
  number={6},
  pages={418--432},
  year={2007},
  publisher={IEEE}
}

@article{raibert_hybrid,
  title={Hybrid Position/Force Control of Manipulators},
  author={Raibert, Marc H and Craig, John J},
  journal={Journal of dynamic systems, measurement, and control},
  volume={103},
  number={2},
  pages={126--133},
  year={1981},
  publisher={American Society of Mechanical Engineers Digital Collection}
}

@inproceedings{tavla,
  title={Elucidating the Design Space of Torque-aware Vision-Language-Action Models},
  author={Zhang, Zongzheng and Xu, Haobo and Yang, Zhuo and Yue, Chenghao and Lin, Zehao and Gao, Huan-ang and Wang, Ziwei and Zhao, Hao},
  booktitle={Conference on Robot Learning},
  year={2025},
  pages = {4019--4037},
  volume = 	 {305},
  publisher =    {PMLR},
}

@inproceedings{forcevla,
  title={ForceVLA: Enhancing VLA Models with a Force-Aware MoE for Contact-Rich Manipulation},
  author={Yu, Jiawen and Liu, Hairuo and Yu, Qiaojun and Ren, Jieji and Hao, Ce and Ding, Haitong and Huang, Guangyu and Huang, Guofan and Song, Yan and Cai, Panpan and others},
  booktitle={Advances in Neural Information Processing Systems},
  year={2025}
}

@inproceedings{pi05,
  title={$\pi_{0.5}$: a Vision-Language-Action Model with Open-World Generalization},
  author={Black, Kevin and Brown, Noah and Darpinian, James and Dhabalia, Karan and Driess, Danny and Esmail, Adnan and Equi, Michael Robert and Finn, Chelsea and Fusai, Niccolo and Galliker, Manuel Y and others},
  booktitle={Conference on Robot Learning},
  year    = {2025},
  pages = {17--40},
  volume = 	 {305},
  publisher =    {PMLR},
}

@inproceedings{forcemimic,
  title={ForceMimic: Force-Centric Imitation Learning With Force-Motion Capture System for Contact-Rich Manipulation},
  author={Liu, Wenhai and Wang, Junbo and Wang, Yiming and Wang, Weiming and Lu, Cewu},
  booktitle={ICRA},
  pages={1105--1112},
  year={2025},
  organization={IEEE}
}

@inproceedings{mimictouch,
  title={MimicTouch: Leveraging Multi-Modal Human Tactile Demonstrations for Contact-Rich Manipulation},
  author={Yu, Kelin and Han, Yunhai and Wang, Qixian and Saxena, Vaibhav and Xu, Danfei and Zhao, Ye},
  booktitle={Conference on Robot Learning},
  year={2024}
}

@inproceedings{maniwav,
    title={ManiWAV: Learning Robot Manipulation from In-the-Wild Audio-Visual Data},
    author={Liu, Zeyi and Chi, Cheng and Cousineau, Eric and Kuppuswamy, Naveen and Burchfiel, Benjamin and Song, Shuran},
    booktitle = {Conference on Robot Learning},
    year={2024}
}

@article{chen2024vegetable,
  title={Vegetable Peeling: A Case Study in Constrained Dexterous Manipulation},
  author={Chen, Tao and Cousineau, Eric and Kuppuswamy, Naveen and Agrawal, Pulkit},
  journal={arXiv preprint arXiv:2407.07884},
  year={2024}
}

@article{hogan1985impedance,
  title={Impedance Control: An Approach to Manipulation},
  author={Hogan, Neville},
  journal={Journal of Dynamic Systems, Measurement, and Control},
  volume={107},
  pages={1--24},
  year={1985}
}

@inproceedings{levine2015learning,
  title        = {Learning Contact-Rich Manipulation Skills With Guided Policy Search},
  author       = {Sergey Levine and
                  Nolan Wagener and
                  Pieter Abbeel},
  booktitle    = {IEEE International Conference on Robotics and Automation},
  pages        = {156--163},
  year         = {2015}
}

@inproceedings{seehearfeel,
  author       = {Hao Li and
                  Yizhi Zhang and
                  Junzhe Zhu and
                  Shaoxiong Wang and
                  Michelle A. Lee and
                  Huazhe Xu and
                  Edward H. Adelson and
                  Li Fei{-}Fei and
                  Ruohan Gao and
                  Jiajun Wu},
  title        = {See, Hear, and Feel: Smart Sensory Fusion for Robotic Manipulation},
  booktitle    = {Conference on Robot Learning},
  pages={1368--1378},
  year         = {2022}
}

@inproceedings{playtothescore,
  title={Play to the Score: Stage-Guided Dynamic Multi-Sensory Fusion for Robotic Manipulation},
  author={Feng, Ruoxuan and Hu, Di and Ma, Wenke and Li, Xuelong},
  booktitle={Conference on Robot Learning},
  year={2024}
}

@article{huang20243d,
  title={3D-ViTac: Learning Fine-Grained Manipulation with Visuo-Tactile Sensing},
  author={Huang, Binghao and Wang, Yixuan and Yang, Xinyi and Luo, Yiyue and Li, Yunzhu},
  journal={arXiv preprint arXiv:2410.24091},
  year={2024}
}

@article{eyesight_hand,
  title={Eyesight Hand: Design of a Fully-Actuated Dexterous Robot Hand with Integrated Vision-based Tactile Sensors and Compliant Actuation},
  author={Romero, Branden and Fang, Hao-Shu and Agrawal, Pulkit and Adelson, Edward},
  journal={arXiv preprint arXiv:2408.06265},
  year={2024}
}

@article{choi2026wild,
  title={In-the-Wild Compliant Manipulation with UMI-FT},
  author={Choi, Hojung and Hou, Yifan and Pan, Chuer and Hong, Seongheon and Patel, Austin and Xu, Xiaomeng and Cutkosky, Mark R and Song, Shuran},
  journal={arXiv preprint arXiv:2601.09988},
  year={2026}
}

@inproceedings{airexo,
  title={AirExo: Low-Cost Exoskeletons for Learning Whole-Arm Manipulation in the Wild},
  author={Fang, Hongjie and Fang, Hao-Shu and Wang, Yiming and Ren, Jieji and Chen, Jingjing and Zhang, Ruo and Wang, Weiming and Lu, Cewu},
  booktitle={IEEE International Conference on Robotics and Automation},
  pages={15031--15038},
  year={2024},
  organization={IEEE}
}

@article{satsevich2025prometheus,
  title={Prometheus: Universal, Open-Source Mocap-Based Teleoperation System with Force Feedback for Dataset Collection in Robot Learning},
  author={Satsevich, Sergei and Bazhenov, Artem and Egorov, Sergei and Erkhov, Artem and Gromakov, Maxim and Fedoseev, Aleksey and Tsetserukou, Dzmitry},
  journal={arXiv preprint arXiv:2510.01023},
  year={2025}
}

@article{bazhenov2025echo,
  title={Echo: An Open-Source, Low-Cost Teleoperation System with Force Feedback for Dataset Collection in Robot Learning},
  author={Bazhenov, Artem and Satsevich, Sergei and Egorov, Sergei and Khabibullin, Farit and Tsetserukou, Dzmitry},
  journal={arXiv preprint arXiv:2504.07939},
  year={2025}
}

@inproceedings{buamanee2024bi,
  title={Bi-ACT: Bilateral Control-Based Imitation Learning via Action Chunking with Transformer},
  author={Buamanee, Thanpimon and Kobayashi, Masato and Uranishi, Yuki and Takemura, Haruo},
  booktitle={IEEE International Conference on Advanced Intelligent Mechatronics},
  pages={410--415},
  year={2024},
  organization={IEEE}
}

@inproceedings{kamijo2024learning,
  title={Learning Variable Compliance Control from a Few Demonstrations for Bimanual Robot with Haptic Feedback Teleoperation System},
  author={Kamijo, Tatsuya and Beltran-Hernandez, Cristian C and Hamaya, Masashi},
  booktitle={IEEE/RSJ International Conference on Intelligent Robots and Systems},
  pages={12663--12670},
  year={2024},
  organization={IEEE}
}

@inproceedings{tacdiffusion,
  title={TacDiffusion: Force-Domain Diffusion Policy for Precise Tactile Manipulation},
  author={Wu, Yansong and Chen, Zongxie and Wu, Fan and Chen, Lingyun and Zhang, Liding and Bing, Zhenshan and Swikir, Abdalla and Haddadin, Sami and Knoll, Alois},
  booktitle={IEEE International Conference on Robotics and Automation},
  pages={11831--11837},
  year={2025},
  organization={IEEE}
}

@inproceedings{admittance_control,
  title={Adaptive Admittance Control: An Approach to Explicit Force Control in Compliant Motion},
  author={Seraji, Homayoun},
  booktitle={IEEE International Conference on Robotics and Automation},
  pages={2705--2712},
  year={1994},
  organization={IEEE}
}

@article{jia2018survey,
  title={A Survey of Automated Threaded Fastening},
  author={Jia, Zhenzhong and Bhatia, Ankit and Aronson, Reuben M and Bourne, David and Mason, Matthew T},
  journal={IEEE Transactions on Automation Science and Engineering},
  volume={16},
  number={1},
  pages={298--310},
  year={2018},
  publisher={IEEE}
}

@article{whitney1987historical,
  title={Historical perspective and state of the art in robot force control},
  author={Whitney, Daniel E},
  journal={The International Journal of Robotics Research},
  volume={6},
  number={1},
  pages={3--14},
  year={1987},
  publisher={Sage Publications Sage CA: Thousand Oaks, CA}
}

@article{de1988compliant,
  title={Compliant Robot Motion I. A Formalism for Specifying Compliant Motion Tasks},
  author={De Schutter, Joris and Van Brussel, Hendrik},
  journal={International Journal of Robotics Research},
  volume={7},
  number={4},
  pages={3--17},
  year={1988},
  publisher={SAGE}
}

@article{buchli2011learning,
  title={Learning Variable Impedance Control},
  author={Buchli, Jonas and Stulp, Freek and Theodorou, Evangelos and Schaal, Stefan},
  journal={International Journal of Robotics Research},
  volume={30},
  number={7},
  pages={820--833},
  year={2011},
  publisher={SAGE}
}

@article{khatib2003unified,
  title={A Unified Approach for Motion and Force Control of Robot Manipulators: The Operational Space Formulation},
  author={Khatib, Oussama},
  journal={IEEE Journal on Robotics and Automation},
  volume={3},
  number={1},
  pages={43--53},
  year={2003},
  publisher={IEEE}
}

@article{chiaverini2002parallel,
  title={The Parallel Approach to Force/Position Control of Robotic Manipulators},
  author={Chiaverini, Stefano and Sciavicco, Lorenzo},
  journal={IEEE Transactions on Robotics and Automation},
  volume={9},
  number={4},
  pages={361--373},
  year={2002},
  publisher={IEEE}
}

@incollection{villani2016force,
  title={Force Control},
  author={Villani, Luigi and De Schutter, Joris},
  booktitle={Springer Handbook of Robotics},
  pages={195--220},
  year={2016},
  publisher={Springer}
}

@article{keemink2018admittance,
  title={Admittance control for physical human--robot interaction},
  author={Keemink, Arvid QL and Van der Kooij, Herman and Stienen, Arno HA},
  journal={The International Journal of Robotics Research},
  volume={37},
  number={11},
  pages={1421--1444},
  year={2018},
  publisher={SAGE Publications Sage UK: London, England}
}

@inproceedings{wang2025sound,
  title={The Sound of Simulation: Learning Multimodal Sim-to-Real Robot Policies with Generative Audio},
  author={Wang, Renhao and Geng, Haoran and Li, Tingle and Wu, Philipp and Wang, Feishi and Anumanchipalli, Gopala and Darrell, Trevor and Li, Boyi and Abbeel, Pieter and Malik, Jitendra and others},
  booktitle={Conference on Robot Learning},
  year    = {2025},
  pages = {420--436},
  volume = 	 {305},
  publisher =    {PMLR},
}

@article{chukwurah2024sim,
  title={Sim-to-Real Transfer in Robotics: Addressing the Gap Between Simulation and Real-World Performance},
  author={Chukwurah, Naomi and Adebayo, Abiodun Sunday and Ajayi, Olanrewaju Oluwaseun},
  journal={International Journal of Robotics and Simulation},
  volume={6},
  number={1},
  pages={89--102},
  year={2024}
}

@article{dexop,
  title={DEXOP: A Device for Robotic Transfer of Dexterous Human Manipulation},
  author={Fang, Hao-Shu and Romero, Branden and Xie, Yichen and Hu, Arthur and Huang, Bo-Ruei and Alvarez, Juan and Kim, Matthew and Margolis, Gabriel and Anbarasu, Kavya and Tomizuka, Masayoshi and others},
  journal={arXiv preprint arXiv:2509.04441},
  year={2025}
}

@article{bruyninckx1995kinematic,
  title={Kinematic Models for Model-Based Compliant Motion in the Presence of Uncertainty},
  author={Bruyninckx, Herman and Demey, Sabine and Dutre, Stefan and De Schutter, Joris},
  journal={International Journal of Robotics Research},
  volume={14},
  number={5},
  pages={465--482},
  year={1995},
  publisher={SAGE}
}

@article{kronander2016stability,
  title={Stability Considerations for Variable Impedance Control},
  author={Kronander, Klas and Billard, Aude},
  journal={IEEE Transactions on Robotics},
  volume={32},
  number={5},
  pages={1298--1305},
  year={2016},
  publisher={IEEE}
}

@article{ureche2015task,
  title={Task Parameterization Using Continuous Constraints Extracted from Human Demonstrations},
  author={Ureche, Ana Lucia Pais and Umezawa, Keisuke and Nakamura, Yoshihiko and Billard, Aude},
  journal={IEEE Transactions on Robotics},
  volume={31},
  number={6},
  pages={1458--1471},
  year={2015},
  publisher={IEEE}
}

@inproceedings{kober2015learning,
  title={Learning Movement Primitives for Force Interaction Tasks},
  author={Kober, Jens and Gienger, Michael and Steil, Jochen J},
  booktitle={IEEE International Conference on Robotics and Automation},
  pages={3192--3199},
  year={2015},
  organization={IEEE}
}

@inproceedings{suomalainen2016learning,
  title={Learning Compliant Assembly Motions from Demonstration},
  author={Suomalainen, Markku and Kyrki, Ville},
  booktitle={IEEE/RSJ International Conference on Intelligent Robots and Systems},
  pages={871--876},
  year={2016},
  organization={IEEE}
}

@inproceedings{conkey2019learning,
  title={Learning Task Constraints from Demonstration for Hybrid Force/Position Control},
  author={Conkey, Adam and Hermans, Tucker},
  booktitle={IEEE-RAS International Conference on Humanoid Robots},
  pages={162--169},
  year={2019},
  organization={IEEE}
}

@inproceedings{suomalainen2017geometric,
  title={A Geometric Approach for Learning Compliant Motions from Demonstration},
  author={Suomalainen, Markku and Kyrki, Ville},
  booktitle={IEEE-RAS International Conference on Humanoid Robotics},
  pages={783--790},
  year={2017},
  organization={IEEE}
}

@article{suomalainen2021imitation,
  title={Imitation Learning-Based Framework for Learning 6-D Linear Compliant Motions},
  author={Suomalainen, Markku and Abu-Dakka, Fares J and Kyrki, Ville},
  journal={Autonomous Robots},
  volume={45},
  number={3},
  pages={389--405},
  year={2021},
  publisher={Springer}
}

@article{camponogara2021integration,
  title={Integration of Haptics and Vision in Human Multisensory Grasping},
  author={Camponogara, Ivan and Volcic, Robert},
  journal={Cortex},
  volume={135},
  pages={173--185},
  year={2021},
  publisher={Elsevier}
}

@article{venkadesan2008neural,
  title={Neural Control of Motion-to-Force Transitions With the Fingertip},
  author={Venkadesan, Madhusudhan and Valero-Cuevas, Francisco J},
  journal={Journal of Neuroscience},
  volume={28},
  number={6},
  pages={1366--1373},
  year={2008},
  publisher={Society for Neuroscience}
}

@article{gharbawie2011cortical,
  title={Cortical Connections of Functional Zones in Posterior Parietal Cortex and Frontal Cortex Motor Regions in New World Monkeys},
  author={Gharbawie, Omar A and Stepniewska, Iwona and Kaas, Jon H},
  journal={Cerebral Cortex},
  volume={21},
  number={9},
  pages={1981--2002},
  year={2011},
  publisher={Oxford University Press}
}

@inproceedings{suomalainen2018learning,
  title={Learning from Demonstration for Hydraulic Manipulators},
  author={Suomalainen, Markku and Koivum{\"a}ki, Janne and Lampinen, Santeri and Kyrki, Ville and Mattila, Jouni},
  booktitle={IEEE/RSJ International Conference on Intelligent Robots and Systems},
  pages={3579--3586},
  year={2018},
  organization={IEEE}
}

@inproceedings{li2023augmentation,
  title={Augmentation Enables One-Shot Generalization in Learning from Demonstration for Contact-Rich Manipulation},
  author={Li, Xing and Baum, Manuel and Brock, Oliver},
  booktitle={IEEE/RSJ International Conference on Intelligent Robots and Systems},
  pages={3656--3663},
  year={2023},
  organization={IEEE}
}

@article{li2022learning,
  title={Learning from Demonstration Based on Environmental Constraints},
  author={Li, Xing and Brock, Oliver},
  journal={IEEE Robotics and Automation Letters},
  volume={7},
  number={4},
  pages={10938--10945},
  year={2022},
  publisher={IEEE}
}

@article{sparc,
  title={On the Analysis of Movement Smoothness},
  author={Balasubramanian, Sivakumar and Melendez-Calderon, Alejandro and Roby-Brami, Agnes and Burdet, Etienne},
  journal={Journal of Neuroengineering and Rehabilitation},
  volume={12},
  number={1},
  pages={112},
  year={2015},
  publisher={Springer}
}

@inproceedings{subramani2018inferring,
  title={Inferring Geometric Constraints in Human Demonstrations},
  author={Subramani, Guru and Zinn, Michael and Gleicher, Michael},
  booktitle={Conference on Robot Learning},
  pages={223--236},
  year={2018},
  organization={PMLR}
}

@inproceedings{perez2017c,
  title={C-Learn: Learning Geometric Constraints from Demonstrations for Multi-Step Manipulation in Shared Autonomy},
  author={P{\'e}rez-D'Arpino, Claudia and Shah, Julie A},
  booktitle={IEEE International Conference on Robotics and Automation},
  pages={4058--4065},
  year={2017},
  organization={IEEE}
}

@article{gr00t,
  title={GR00T N1: An Open Foundation Model for Generalist Humanoid Robots},
  author={Bjorck, Johan and Casta{\~n}eda, Fernando and Cherniadev, Nikita and Da, Xingye and Ding, Runyu and Fan, Linxi and Fang, Yu and Fox, Dieter and Hu, Fengyuan and Huang, Spencer and others},
  journal={arXiv preprint arXiv:2503.14734},
  year={2025}
}

@inproceedings{umi,
  author       = {Cheng Chi and
                  Zhenjia Xu and
                  Chuer Pan and
                  Eric Cousineau and
                  Benjamin Burchfiel and
                  Siyuan Feng and
                  Russ Tedrake and
                  Shuran Song},
  title        = {Universal Manipulation Interface: In-The-Wild Robot Teaching Without
                  In-The-Wild Robots},
  booktitle    = {Robotics: Science and Systems},
  year         = {2024}
}

@inproceedings{cage,
  title={CAGE: Causal Attention Enables Data-Efficient Generalizable Robotic Manipulation},
  author={Xia, Shangning and Fang, Hongjie and Lu, Cewu and Fang, Hao-Shu},
  booktitle={IEEE International Conference on Robotics and Automation},
  pages={13242--13249},
  year={2025},
  organization={IEEE}
}

@misc{gemini3pro,
  author       = {Google DeepMind},
  title        = {Gemini 3 Pro},
  note         = {Accessed January 2026}
}

@misc{tdk,
  author = {Flexiv Ltd},
  title = {Flexiv Teleoperation Development Kit (TDK)},
  note = {Accessed January 2026}
}

@article{bruyninckx1996specification,
  title={Specification of Force-Controlled Actions in the "Task Frame Formalism"  - a Synthesis},
  author={Bruyninckx, Herman and De Schutter, Joris},
  journal={IEEE Transactions on Robotics and Automation},
  volume={12},
  number={4},
  pages={581--589},
  year={1996},
  publisher={IEEE}
}

@article{dipcom,
  title={Learning Diffusion Policies from Demonstrations for Compliant Contact-Rich Manipulation},
  author={Aburub, Malek and Beltran-Hernandez, Cristian C and Kamijo, Tatsuya and Hamaya, Masashi},
  journal={arXiv preprint arXiv:2410.19235},
  year={2024}
}

@article{goldring2022functional,
  title={Functional Characterization of the Fronto-Parietal Reaching and Grasping Network: Reversible Deactivation of M1 and Areas 2, 5, and 7b in Awake Behaving Monkeys},
  author={Goldring, Adam B and Cooke, Dylan F and Pineda, Carlos R and Recanzone, Gregg H and Krubitzer, Leah A},
  journal={Journal of Neurophysiology},
  volume={127},
  number={5},
  pages={1363--1387},
  year={2022},
  publisher={American Physiological Society}
}

@inproceedings{sail,
  title = 	 {SAIL: Faster-than-Demonstration Execution of Imitation Learning Policies},
  author =       {Arachchige, Nadun Ranawaka and Chen, Zhenyang and Jung, Wonsuhk and Shin, Woo Chul and Bansal, Rohan and Barroso, Pierre and He, Yu Hang and Lin, Yingyan Celine and Joffe, Benjamin and Kousik, Shreyas and Xu, Danfei},
  booktitle = 	 {Conference on Robot Learning},
  pages = 	 {721--749},
  year = 	 {2025},
  volume = 	 {305},
  publisher =    {PMLR},
}

@inproceedings{rtc,
title={Real-Time Execution of Action Chunking Flow Policies},
author={Kevin Black and Manuel Y Galliker and Sergey Levine},
booktitle={Advances in Neural Information Processing Systems},
year={2025}
}

@article{training_rtc,
  title={Training-Time Action Conditioning for Efficient Real-Time Chunking},
  author={Black, Kevin and Ren, Allen Z and Equi, Michael and Levine, Sergey},
  journal={arXiv preprint arXiv:2512.05964},
  year={2025}
}

@Inbook{dtw,
title={Dynamic Time Warping},
author={Meinard Müller},
bookTitle={Information Retrieval for Music and Motion},
year={2007},
publisher={Springer},
pages={69--84}
}

@article{liu2025immimic,
  title={ImMimic: Cross-Domain Imitation from Human Videos via Mapping and Interpolation},
  author={Liu, Yangcen and Shin, Woo Chul and Han, Yunhai and Chen, Zhenyang and Ravichandar, Harish and Xu, Danfei},
  journal={arXiv preprint arXiv:2509.10952},
  year={2025}
}

@ARTICLE{1163491,
  author={Myers, C. and Rabiner, L. and Rosenberg, A.},
  title={Performance Tradeoffs in Dynamic Time Warping Algorithms for Isolated Word Recognition}, 
  journal={IEEE Transactions on Acoustics, Speech, and Signal Processing}, 
  year={1980},
  volume={28},
  number={6},
  pages={623-635}
}

@ARTICLE{1163055,
  author={Sakoe, H. and Chiba, S.},
  title={Dynamic Programming Algorithm Optimization for Spoken Word Recognition}, 
  journal={IEEE Transactions on Acoustics, Speech, and Signal Processing}, 
  year={1978},
  volume={26},
  number={1},
  pages={43-49}
}

@inproceedings{issam2025dtwalignbridgingmodalitygap,
  title={DTW-Align: Bridging the Modality Gap in End-to-End Speech Translation with Dynamic Time Warping Alignment},
  author={Issam, Abderrahmane and Semerci, Yusuf Can and Scholtes, Jan and Spanakis, Gerasimos},
  booktitle={Proceedings of the Tenth Conference on Machine Translation},
  pages={191--199},
  year={2025}
}

@article{mip,
  title={Much Ado About Noising: Dispelling the Myths of Generative Robotic Control},
  author={Pan, Chaoyi and Anantharaman, Giri and Huang, Nai-Chieh and Jin, Claire and Pfrommer, Daniel and Yuan, Chenyang and Permenter, Frank and Qu, Guannan and Boffi, Nicholas and Shi, Guanya and others},
  journal={arXiv preprint arXiv:2512.01809},
  year={2025}
}

@inproceedings{film,
  title={FiLM: Visual Reasoning with a General Conditioning Layer},
  author={Perez, Ethan and Strub, Florian and De Vries, Harm and Dumoulin, Vincent and Courville, Aaron},
  booktitle={AAAI conference on Artificial Intelligence},
  pages = {3942--3951},
  year={2018}
}

@inproceedings{resnet,
  title={Deep Residual Learning for Image Recognition},
  author={He, Kaiming and Zhang, Xiangyu and Ren, Shaoqing and Sun, Jian},
  booktitle={IEEE Conference on Computer Vision and Pattern Recognition},
  pages={770--778},
  year={2016}
}

@inproceedings{gru,
  author       = {Kyunghyun Cho and
                  Bart van Merrienboer and
                  {\c{C}}aglar G{\"{u}}l{\c{c}}ehre and
                  Dzmitry Bahdanau and
                  Fethi Bougares and
                  Holger Schwenk and
                  Yoshua Bengio},
  title        = {Learning Phrase Representations using {RNN} Encoder-Decoder for Statistical Machine Translation},
  booktitle    = {Empirical Methods in Natural Language Processing},
  pages        = {1724--1734},
  publisher    = {{ACL}},
  year         = {2014}
}

\clearpage
\begin{center}
    \Large \textbf{Supplementary Materials for \textit{Force Policy}}
\end{center}
\setcounter{section}{0}
\bookmarksetup{startatroot}
\renewcommand{\theHsection}{Supp.~\Roman{section}}
\renewcommand{\thesection}{Supp.~\Roman{section}}
\renewcommand{\thesubsection}{\thesection.\arabic{subsection}}

\section{Theoretical Formulation}

\subsection{Geometry-Induced Interaction Structure}

Let $\delta\boldsymbol{\chi} \in \mathfrak{se}(3)$ denote the infinitesimal pose perturbation at the interaction point $p_I$. The resulting interaction wrench $\boldsymbol{\mathcal{W}}_{\text{tot}} \in \mathfrak{se}^*(3)$ measured at $p_I$ is generally a superposition of conservative and dissipative effects. To extract the geometric topology, we apply an additive decomposition:
\begin{equation}
\boldsymbol{\mathcal{W}}_{\text{tot}} = \boldsymbol{\mathcal{W}}_{\text{el}} + \boldsymbol{\mathcal{W}}_{\text{dis}}.
\end{equation}
Here, $\boldsymbol{\mathcal{W}}_{\text{dis}}$ denotes non-conservative terms like friction. Crucially, we postulate that the structural constraints arise exclusively from an internal elastic potential energy $U(\delta\boldsymbol{\chi})$. Thus, the conservative elastic wrench $\boldsymbol{\mathcal{W}}_{\text{el}}$ is defined as the negative gradient of this potential:
\begin{equation}
\boldsymbol{\mathcal{W}}_{\text{el}}(\delta{\boldsymbol{\chi}}) \triangleq -\nabla U(\delta\boldsymbol{\chi}).
\end{equation}

The environmental stiffness $\mathbf{K}_{\mathrm{env}}$ is formally defined as the sensitivity of the elastic wrench to the pose perturbation:
\begin{equation}
\mathbf{K}_{\mathrm{env}}
\triangleq
-\left.\frac{\partial \boldsymbol{\mathcal W}_{\mathrm{el}}(\delta\boldsymbol{\chi})}{\partial (\delta\boldsymbol{\chi})}\right|_{\delta\boldsymbol{\chi}=\mathbf 0}=\left.\nabla^2 U(\delta\boldsymbol{\chi})\right|_{\delta\boldsymbol{\chi}=\mathbf 0}.
\end{equation}

Since $\mathbf{K}_{\mathrm{env}}$ is modeled as the Hessian of a conservative potential, it is \textit{symmetric}. Then, we derive the following theorem, showing that the translational and rotational stiffness in $\mathbf{K}_\text{env}$ can be diagonalized using the same rotation matrix.

\begin{theorem}[Unified Spatial Basis]\label{thm:diag}
Under the assumption that the local interaction follows the Hertzian contact model, there exists a single rotation matrix $\mathbf{R} \in SO(3)$ such that the corresponding spatial rotation $\mathbf{\Phi} = \operatorname{diag}(\mathbf{R}, \mathbf{R}) \in \mathbb{R}^{6 \times 6}$ diagonalizes the full environment stiffness matrix $\mathbf{K}_{\mathrm{env}} \in \mathbb{R}^{6 \times 6}$.
\end{theorem}

\noindent \textbf{\textit{Proof}.} We explicitly model the local contact interface at the interaction point $p_I$ according to Hertzian contact theory. Under this formulation, the local surface separation is approximated by a quadratic form, resulting in an elliptical contact patch $\mathcal{D} \subset \mathbb{R}^2$. We express wrenches/twists at the stiffness centroid of the contact patch, \textit{i.e.}, the first moments of the stiffness distribution over $\mathcal D$ vanish. We construct the Principal Geometric Frame, denoted as $\Sigma_{\text{geo}}$, defined by the orthonormal basis $\{ \mathbf{t}_1, \mathbf{t}_2, \mathbf{n} \}$, where $\mathbf{n}$ is the common surface normal, and $\mathbf{t}_1, \mathbf{t}_2$ align with the principal axes of the contact ellipse (corresponding to the principal relative curvature directions). By the definition of Hertzian geometry, the domain $\mathcal{D}$ possesses intrinsic reflectional symmetry with respect to the planes defined by $\mathbf{n}$-$\mathbf{t}_1$ and $\mathbf{n}$-$\mathbf{t}_2$. We define the local coordinates $(x, y)$ along $(\mathbf{t}_1, \mathbf{t}_2)$.

We model the contact elasticity using a distributed Winkler foundation model, where the stiffness densities $k_{t_1}(x, y), k_{t_2}(x,y)$ and $k_n(x,y)$ are distributed over $\mathcal{D}$. The translational stiffness matrix in the geometric frame, $\mathbf{K}_v^{\text{geo}}$, represents the resistance to linear displacements. The stiffness elements are therefore:
\begin{equation}
\begin{aligned}
    k_{xx} &= \int_{\mathcal{D}} k_{t_1}(x,y) \, \mathrm{d}x\mathrm{d}y, \\
    k_{yy} &= \int_{\mathcal{D}} k_{t_2}(x,y) \, \mathrm{d}x\mathrm{d}y, \\
    k_{zz} &= \int_{\mathcal{D}} k_n(x,y) \, \mathrm{d}x\mathrm{d}y.
\end{aligned}
\end{equation}
Due to the orthogonality of the basis vectors $\{ \mathbf{t}_1, \mathbf{t}_2, \mathbf{n} \}$, the linear resistance is decoupled along these axes (\textit{e.g.}, a pure normal displacement does not induce a net shear force in a symmetric isotropic contact). Thus, $\mathbf{K}_v^{\text{geo}}$ is diagonal:
\begin{equation}\label{eq:v-term}
\mathbf{K}_v^{\text{geo}} = \operatorname{diag}(k_{xx}, k_{yy}, k_{zz}).
\end{equation}

The rotational stiffness arises from the moment of the distributed elastic forces. The stiffness elements correspond to the second moments of the stiffness distribution. Consider the rotational stiffness components in $\Sigma_{\text{geo}}$. For principal stiffnesses, we obtain:
\begin{equation}
\begin{aligned}
    \kappa_{xx} &= \int_{\mathcal{D}} k_n(x,y) y^2 \, \mathrm{d}x\mathrm{d}y, \\
    \kappa_{yy} &= \int_{\mathcal{D}} k_n(x,y) x^2 \, \mathrm{d}x\mathrm{d}y, \\
    \kappa_{zz} &= \int_{\mathcal{D}} \left(k_{t_2}(x,y)x^2 + k_{t_1}(x,y)y^2\right) \, \mathrm{d}x\mathrm{d}y.
\end{aligned}
\end{equation}

To establish the diagonality of $\mathbf{K}_\omega^{\text{geo}}$, we verify that all off-diagonal terms vanish due to the symmetry of the contact domain $\mathcal{D}$. First, the in-plane coupling $\kappa_{xy}$ vanishes because the geometric term $xy$ is an odd function, while the stiffness distribution $k_n(x,y)$ is even:
\begin{equation}
\kappa_{xy} = -\int_{\mathcal{D}} k_n(x,y)xy \, \mathrm{d}x\mathrm{d}y = 0.
\end{equation}
Similarly, the out-of-plane couplings $\kappa_{xz}$ and $\kappa_{yz}$ vanish. Under the Hertzian assumption, the surface profile $z(x,y)$ is quadratic and thus an even function. Moreover, under macroscopic reflectional symmetry, the stiffness densities are even functions over $\mathcal D$. Consequently, the integrands involving the moment arms $x$ and $y$ become odd functions, yielding:
\begin{equation}
\begin{aligned}
\kappa_{xz} &= -\int_{\mathcal{D}} k_{t_1}(x,y) z(x,y)x \, \mathrm{d}x\mathrm{d}y = 0, \\
\kappa_{yz} &= -\int_{\mathcal{D}} k_{t_2}(x,y) z(x,y)y \, \mathrm{d}x\mathrm{d}y = 0.
\end{aligned}
\end{equation}

Therefore, $\mathbf{K}_\omega^\text{geo}$ is diagonal:
\begin{equation}\label{eq:omega-term}
    \mathbf{K}_\omega^\text{geo} =  \operatorname{diag}(\kappa_{xx}, \kappa_{yy}, \kappa_{zz}).
\end{equation}

Finally, we address the linear-angular coupling blocks $\mathbf{K}_{v\omega}$. These terms represent the net force induced by a pure rotation (or net torque by a pure translation) and involve the first moments of the stiffness distribution. For instance, the normal force induced by a rotation about the $y$-axis is proportional to:
\begin{equation}\label{eq:cross-term}
k_{z\theta_y} = \int_{\mathcal{D}} k_n(x,y) x \, \mathrm{d}x\mathrm{d}y = 0,
\end{equation}
since $k_n(x,y)$ is even and $x$ is odd. By the same symmetry argument applied to all first-moment integrals over $\mathcal{D}$, the coupling blocks vanish ($\mathbf{K}_{v\omega} = \mathbf{0}$).

Combining Eqn.~\eqref{eq:v-term}, \eqref{eq:omega-term}, and \eqref{eq:cross-term}, we confirm that in the principal geometric frame $\Sigma_{\text{geo}}$, the stiffness matrix $\mathbf{K}_\text{env}$ is fully diagonalized by $\mathbf{\Phi} = \operatorname{diag}(\mathbf{R}, \mathbf{R})$.
\hfill\qed
\vspace{0.3cm}

\noindent \textbf{Remark }(Generalization to Multi-Point and Area Contacts). Although Theorem~\ref{thm:diag} is derived explicitly under the Hertzian contact formulation, the resulting diagonalization property extends to a broader class of interactions relevant to robotics, such as multi-point contacts and planar area contacts. The sufficient condition for the decoupling is the \textit{macroscopic reflectional symmetry} of the stiffness distribution and the contact domain $\mathcal{D}$. For instance, in peg-in-hole assembly, the interaction often manifests as a ring contact or a symmetric multi-point pattern. In polishing, the tool-surface interface may form a planar patch. In these scenarios, the contact domain $\mathcal{D}$ retains a center of symmetry, and the effective stiffness distribution remains an even function with respect to the principal axes. Consequently, the parity arguments used in the proof remain valid, ensuring that the spatial axes defined by the macroscopic geometry still diagonalize the stiffness matrix.

\vspace{0.2cm}

On the other hand, by the spectral theorem, the symmetric $\mathbf{K}_{\mathrm{env}}$ has real eigenvalues and orthonormal eigenvectors:
\begin{equation}
\mathbf K_{\mathrm{env}}(p_I)=\mathbf Q\boldsymbol{\Lambda}\mathbf Q^\top
\end{equation}
where $\boldsymbol{\Lambda}=\mathrm{diag}(\lambda_1,\cdots,\lambda_6)$ denotes the eigenvalues, and
$\mathbf Q=[\boldsymbol{q}_1,\cdots,\boldsymbol{q}_6]$ is orthonormal. The eigenvalues of $\mathbf{K}_{\mathrm{env}}$ can be partitioned into the \textbf{constraint subspace} $\mathcal U$ with high stiffness $\lambda_i \gg 0$ and the \textbf{admissible-motion subspace} $\mathcal T$ with negligible stiffness $\lambda_i \approx 0$, \textit{i.e.}, for a given threshold $\epsilon>0$,
\begin{equation}
    \mathcal{U} \triangleq \operatorname{span}\{\boldsymbol{q}_i:\lambda_i > \epsilon\}, \quad \mathcal{T}\triangleq \operatorname{span}\{\boldsymbol{q}_i:\lambda_i \leq \epsilon\}
\end{equation}

Compared with Theorem~\ref{thm:diag}, we can derive the following corollary about spectral-geometric isomorphism.

\begin{corollary}[Spectral-Geometric Isomorphism]\label{cor:subspace}
The eigenbasis $\mathbf{Q}$ of $\mathbf{K}_{\mathrm{env}}$ is aligned with the principal axes of the geometric frame $\Sigma_{\text{geo}}$. Consequently, the spectral components map directly to the physical stiffness properties:
\begin{enumerate}
    \item \textbf{Eigenvalues ($\boldsymbol{\Lambda}$):} The diagonal elements of $\boldsymbol{\Lambda}$ are identically the distributed stiffness integrals derived in the Hertzian model. Specifically, up to a permutation, the spectrum is the set:
    $$
        \sigma(\mathbf{K}_{\mathrm{env}}) = \{ k_{xx}, k_{yy}, k_{zz}, \kappa_{xx}, \kappa_{yy}, \kappa_{zz} \}.
    $$
    \item \textbf{Constraint Subspace ($\mathcal{U}$):} Defined by eigenvectors with significant stiffness eigenvalues ($\lambda_i > \epsilon$), $\mathcal{U}$ corresponds to geometric directions of hard contact (e.g., $\mathbf{n}$).
    $$
        \mathcal{U} \equiv \operatorname{span}\{ \mathbf{v} \in \Sigma_{\text{geo}} \mid \text{stiffness along } \mathbf{v} \text{ is dominant} \}.
    $$
    \item \textbf{Admissible Motion Subspace ($\mathcal{T}$):} Defined by eigenvectors with negligible stiffness ($\lambda_i \leq \epsilon$), $\mathcal{T}$ corresponds to directions of free motion or sliding (e.g., $\mathbf{t}_{1,2}$).
\end{enumerate}
\end{corollary}

\subsection{Task Intent and Interaction Frame}

While the spectral analysis identifies the geometric orientation of the constraints, it leaves the directionality and semantic assignment of the axes ambiguous. To resolve this, we introduce the task intent as the causal anchor for the interaction frame.
Let the \textbf{task intent} $\mathcal{I}$ be defined by the tuple of the desired interaction wrench and twist at $p_I$:
\begin{equation}
\mathcal{I} \triangleq \{ \boldsymbol{\xi}^* \in \mathfrak{se}(3) ,\; \boldsymbol{\mathcal{W}}^* \in \mathfrak{se}^*(3)\}.
\end{equation}

\begin{assumption}[Intent Compatibility]\label{ass:intent}
Assume that the task is physically consistent and non-destructive. Specifically, the agent intends to exert forces primarily against environmental constraints and execute motions primarily along admissible freedoms. Mathematically, this implies that the task intent lies within the corresponding spectral subspaces:
\begin{equation}
\boldsymbol{\mathcal{W}}^* \in \mathcal{U} \quad \text{and} \quad \boldsymbol{\xi}^* \in \mathcal{T}.
\end{equation}
Since $\mathcal{U} \perp \mathcal{T}$ (due to the symmetry of $\mathbf{K}_{\mathrm{env}}$), the force and motion intents are orthogonal: $\langle \boldsymbol{\mathcal{W}}^*, \boldsymbol{\xi}^* \rangle = 0$. Further, we assume that for anisotropic constraints, \textbf{meaningful} task intent naturally aligns with the principal geometric axes. For isotropic constraints like spherical contact, any intent direction within the subspace is valid.
\end{assumption}

To resolve the directional ambiguity and handle co-axial degeneracies, we construct the \textbf{Interaction Frame (IF) $\Sigma(p_I)$} by anchoring the spectral axes to the task intent. We pre-define a reference vector $\mathbf{v}_{\text{ref}}$ like the current end-effector $x$-axis to resolve ambiguities. Next, for the intended wrench and twist, they admit a well-defined screw-axis \textit{direction}~\cite{screw_theory}, represented as $\mathbf{n}_{\boldsymbol{\mathcal{W}}^*}$ and $\mathbf{n}_{\boldsymbol{\xi}^*}$, respectively. Notice that we only use the axis \textit{direction} for frame construction, and the axis location is irrelevant here. The basis vectors $\{ \mathbf{x}_{\Sigma}, \mathbf{y}_{\Sigma}, \mathbf{z}_{\Sigma} \}$ are derived via the following prioritized scheme:

\begin{itemize}
\item \textbf{Primary Axis Assignment}: The frame definition prioritizes the wrench intent. (1) If the wrench intent is non-negligible, it defines the \textbf{constraint axis}: $\mathbf{z}_{\Sigma} \triangleq \mathbf{n}_{\boldsymbol{\mathcal{W}}^*}$; (2) otherwise, the twist intent takes precedence to define the \textbf{motion axis}: $\mathbf{x}_{\Sigma} \triangleq \mathbf{n}_{\boldsymbol{\xi}^*}$.

\item \textbf{Secondary Axis with Fallback}: The remaining axis is determined by projecting the secondary intent onto the subspace orthogonal to the primary axis. If the secondary intent is \textit{degenerate}, \textit{negligible}, or \textit{collinear with the primary axis}, it fails to define a unique orthogonal direction. In such cases, the definition falls back to the reference vector $\mathbf{v}_{\text{ref}}$. For instance, if $\mathbf{z}_{\Sigma}$ is the primary axis, the motion axis $\mathbf{x}_{\Sigma}$ is computed as $\mathbf{p}/\left\| \mathbf{p} \right\|$, where vector $\mathbf{p}$ selects the most informative source:
$$
\mathbf{p}= 
\begin{cases} 
(\mathbf{I} - \mathbf{z}_{\Sigma}\mathbf{z}_{\Sigma}^\top)\mathbf{n}_{\boldsymbol{\xi}^*} & \text{if } \|\boldsymbol{\xi}^*\| > \epsilon_\xi \text{ and } |\mathbf{z}_{\Sigma} \cdot \mathbf{n}_{\boldsymbol{\xi}^*}| < 1 - \epsilon_\parallel, \\
(\mathbf{I} - \mathbf{z}_{\Sigma}\mathbf{z}_{\Sigma}^\top)\mathbf{v}_{\text{ref}} & \text{otherwise (fallback)}.
\end{cases}
$$
where $\epsilon_\xi$ and $\epsilon_\parallel$ are thresholds for twist and collinearity.

\item \textbf{Completion}: The final axis $\mathbf{y}_{\Sigma}$ is determined via the right-hand rule to complete the orthonormal basis.

\end{itemize}

Thus, we show that the constructed IF can also diagonalizes the environmental stiffness:

\begin{proposition}[Intent Alignment]\label{prop:alignment}Under Assumption~\ref{ass:intent}, the Interaction Frame $\Sigma(p_I)$ is spatially co-axial with the Principal Geometric Frame $\Sigma_{\text{geo}}$. Consequently, the matrix $\mathbf{\Phi}_{\Sigma} = \operatorname{diag}(\mathbf{R}_{\Sigma}, \mathbf{R}_{\Sigma})$ formed by the basis of $\Sigma(p_I)$ diagonalizes the environmental stiffness $\mathbf{K}_{\mathrm{env}}$.
\end{proposition}

\noindent\textbf{\textit{Proof (Sketch)}.} Under the assumption, the task intent naturally aligns with the environment's structural topology: wrenches are exerted against constraints ($\mathcal{U}$), and motions occur along admissible freedoms ($\mathcal{T}$). In scenarios where these subspaces are multi-dimensional (\textit{e.g.}, the isotropic radial constraint in peg-in-hole assembly), the stiffness distribution exhibits rotational symmetry, implying that any direction selected by the intent within that subspace constitutes a mathematically valid principal axis. Consequently, the IF constructed from the task intent does not impose an arbitrary structure; rather, it instantiates a specific, physically valid eigenbasis that simultaneously diagonalizes the environmental stiffness matrix and resolves the sign and gauge ambiguities inherent in purely geometric analysis.\hfill\qed
\vspace{0.3cm}

\begin{theorem}[Co-axial Twist to an Applied Wrench]\label{thm:force-to-motion}
Let $\Sigma(p_I)=\{\mathbf e_1,\dots,\mathbf e_6\}$ be the Interaction Frame (IF) constructed above, and let
$\mathbf{\Phi}_\Sigma=\operatorname{diag}(\mathbf R_\Sigma,\mathbf R_\Sigma)$ be the corresponding spatial rotation.
In the IF, the elastic environment stiffness is diagonal,
$\mathbf K_{\mathrm{env}}^\Sigma=\mathbf{\Phi}_\Sigma^\top \mathbf K_{\mathrm{env}} \mathbf{\Phi}_\Sigma
=\operatorname{diag}(k_1,\dots,k_6)$.
Then, for any applied wrench $\boldsymbol{\mathcal W}^*$ aligned with a constraint-axis $\mathbf e_j\in\mathcal U$,
the induced elastic compliance is co-axial: the resulting pose perturbation $\delta\boldsymbol{\chi}_c$ and the
corresponding parasitic twist $\boldsymbol{\xi}_c$ are strictly collinear with $\mathbf e_j$.
\end{theorem}

\noindent\textbf{\textit{Proof}.} 
By Proposition~\ref{prop:alignment}, the basis vectors of the Interaction Frame align with the principal axes of $\Sigma_{\text{geo}}$. Thus, invoking Theorem~\ref{thm:diag}, the stiffness matrix $\mathbf{K}_{\mathrm{env}}$ is diagonal in this frame:
\begin{equation}
\mathbf K_{\mathrm{env}}^\Sigma=\operatorname{diag}(k_1,\dots,k_6),
\end{equation}
where $k_i$ is the stiffness eigenvalue associated with $\mathbf e_i$.

For small perturbations, the elastic constitutive relation is
\begin{equation}
\boldsymbol{\mathcal W}_{\mathrm{el}}=\mathbf K_{\mathrm{env}}^\Sigma\,\delta\boldsymbol{\chi}.
\end{equation}
Consider an applied wrench aligned with the $j$-th constraint basis vector:
$\boldsymbol{\mathcal W}^*=\beta\,\mathbf e_j$ with $\mathbf e_j\in\mathcal U$ and $\beta\in\mathbb R$.
We seek the elastic compliance $\delta\boldsymbol{\chi}_c$ satisfying
$\mathbf K_{\mathrm{env}}^\Sigma\,\delta\boldsymbol{\chi}_c=\boldsymbol{\mathcal W}^*$.
Since $\mathbf e_j\in\mathcal U$, by definition its associated stiffness satisfies $k_j>0$, hence the scalar equation
$k_j\,\delta\chi_j=\beta$ admits the unique solution $\delta\chi_j=\beta/k_j$.
Therefore,
\begin{equation}
\delta\boldsymbol{\chi}_c=\left(\frac{\beta}{k_j}\right)\mathbf e_j,
\end{equation}
and substituting back verifies
\begin{equation}
\mathbf K_{\mathrm{env}}^\Sigma\,\delta\boldsymbol{\chi}_c
=\left(\frac{\beta}{k_j}\right)\mathbf K_{\mathrm{env}}^\Sigma \mathbf e_j
=\left(\frac{\beta}{k_j}\right)(k_j\mathbf e_j)
=\beta\mathbf e_j
=\boldsymbol{\mathcal W}^*.
\end{equation}
Thus, $\delta\boldsymbol{\chi}_c$ is strictly collinear with the applied wrench. Assuming the control loop or integration preserves directionality, the resulting twist $\boldsymbol{\xi}_c$ aligns with $\delta \boldsymbol{\chi}_c$, thus $\boldsymbol{\xi}_c\parallel \mathbf{e}_j$.\hfill\qed
\vspace{0.3cm}

\begin{proposition}[Co-axial Wrench to an Applied Twist]\label{prop:motion-to-force}
In the IF $\Sigma(p_I)$, consider an intended twist $\boldsymbol{\xi}^*=\alpha\,\mathbf e_i$ along a principal axis $\mathbf e_i\in\mathcal T$. 
If the local dissipative mechanism is isotropic and the contact patch (or its macroscopic stiffness/friction distribution) is reflectionally symmetric with respect to the plane spanned by $\mathbf e_i$ and the contact normal, then the induced dissipative wrench must be co-axial:
\begin{equation}
\boldsymbol{\mathcal W}_c = \boldsymbol{\mathcal W}_{\mathrm{dis}}(\boldsymbol{\xi}^*) \parallel \boldsymbol{\xi}^*,
\qquad 
\boldsymbol{\mathcal W}_c^\top \boldsymbol{\xi}^* \le 0 .
\end{equation}
\end{proposition}

\noindent\textbf{\textit{Proof (Sketch)}.} Under the stated symmetries, there is no distinguished lateral direction orthogonal to $\mathbf e_i$ that can bias the dissipative response. 
Hence any lateral component of $\boldsymbol{\mathcal W}_c$ would break reflectional symmetry and is therefore forbidden; the only admissible direction is the axis of motion $\mathbf e_i$. 
Moreover, dissipation opposes motion, yielding $\boldsymbol{\mathcal W}_c^\top \boldsymbol{\xi}^*\le 0$.
\hfill\qed \vspace{0.3cm}

Therefore, the intended twist $\boldsymbol{\xi}^*$ may induce a parasitic wrench $\boldsymbol{\mathcal{W}}_c$ along its direction (Proposition~\ref{prop:motion-to-force}), and the intended wrench $\boldsymbol{\mathcal W}^*$ may induce a parasitic twist $\boldsymbol{\xi}_c$ along its direction (Theorem~\ref{thm:force-to-motion}). Since $\boldsymbol{\mathcal{W}}_c \parallel \boldsymbol{\xi}^* \perp \boldsymbol{\mathcal{W}}^* \parallel \boldsymbol{\xi}_c$,
\begin{equation}
    \boldsymbol{P} = (\boldsymbol{\mathcal{W}}^* + \boldsymbol{\mathcal{W}}_c)^\top (\boldsymbol{\xi}^* + \boldsymbol{\xi}_c) = \boldsymbol{\mathcal{W}}_c^\top \boldsymbol{\xi}^* + \boldsymbol{\mathcal{W}}^{*\top} \boldsymbol{\xi}_c,
\end{equation}
which establishes the foundation of the approximation strategy described in the paper.

\section{Interaction Structure Discovery}

\begin{figure}[b]
    \centering
    \includegraphics[width=\linewidth]{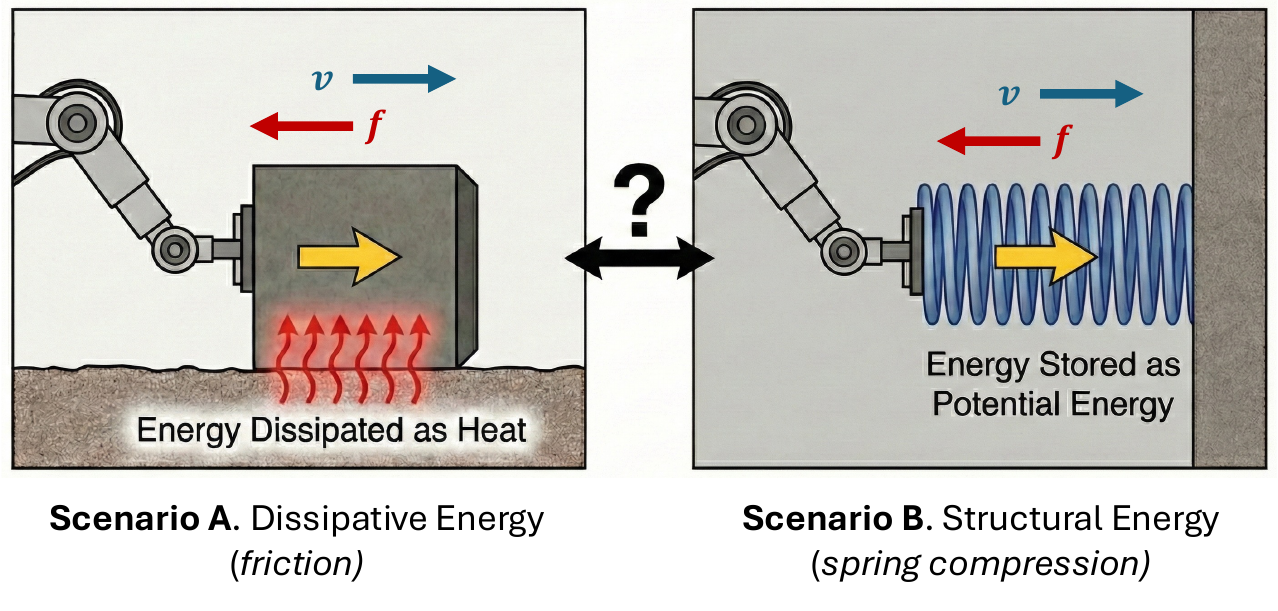}
    \caption{\textbf{Sensor Signal Ambiguity.} Similar sensor signals (wrench and twist) can lead to different interaction types and structures.}
    \label{fig:high-level} \vspace{-0.3cm}
\end{figure}

\begin{figure*}
    \centering
    \includegraphics[width=\linewidth]{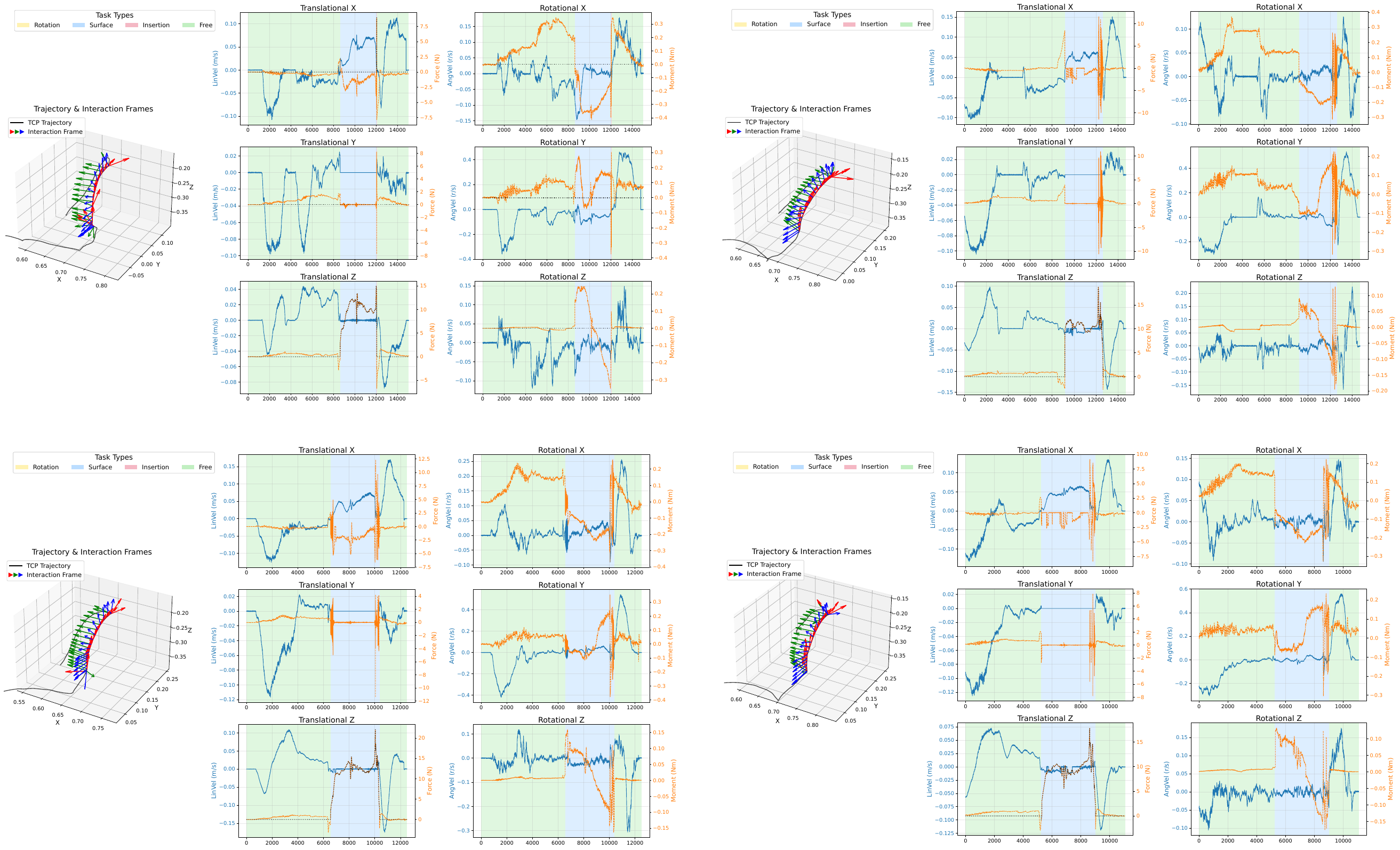}
    \caption{\textbf{Visualization of the Interaction Frame and the Task Mode on the \textit{Push and Flip} Task.}}
    \label{fig:vis-flip}
\end{figure*}

\begin{figure*}
    \centering
    \includegraphics[width=\linewidth]{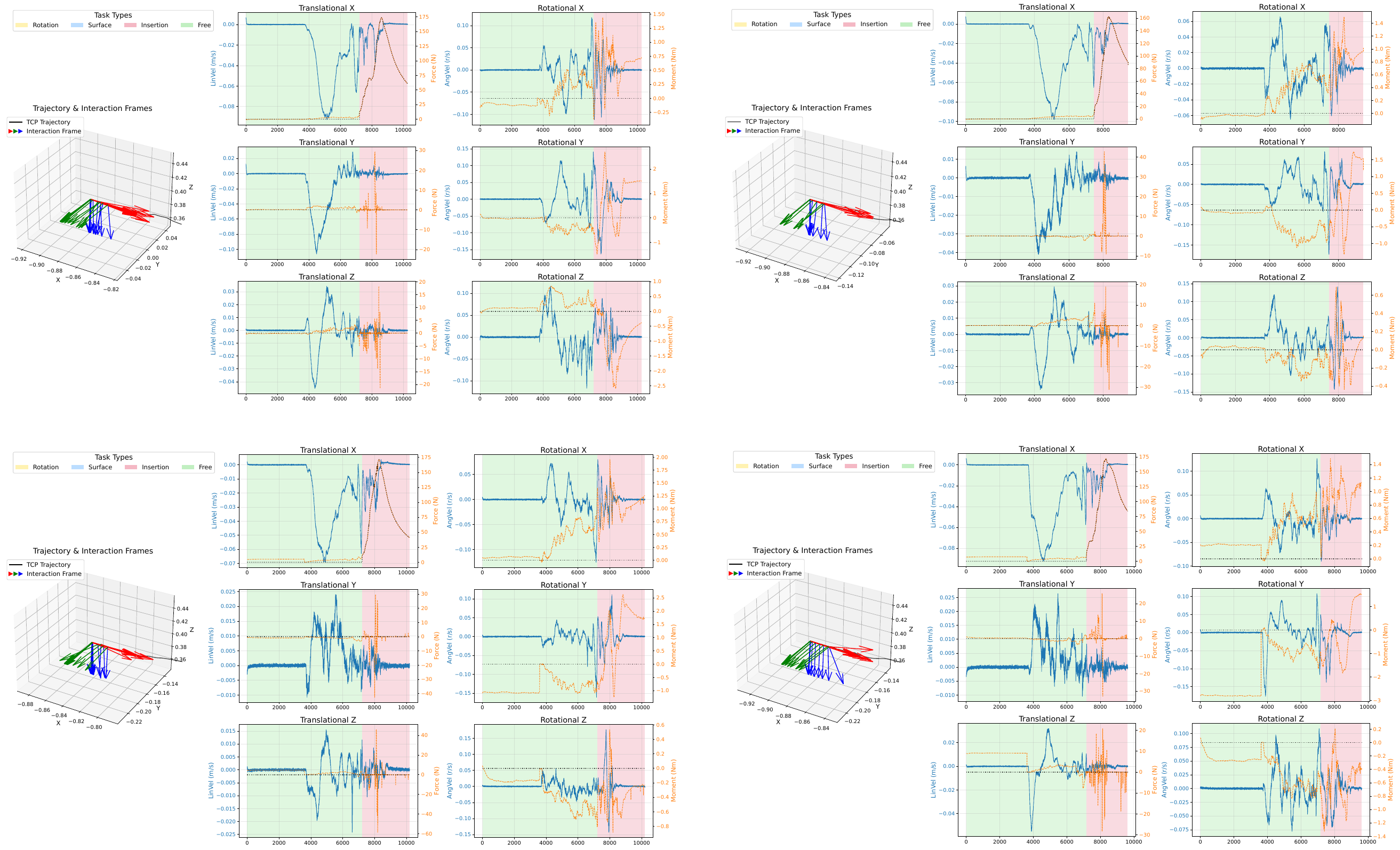}
    \caption{\textbf{Visualization of the Interaction Frame and the Task Mode on the \textit{Plug in EV Charger} Task.}}
    \label{fig:vis-charger}
\end{figure*}

\begin{figure*}
    \centering
    \includegraphics[width=\linewidth]{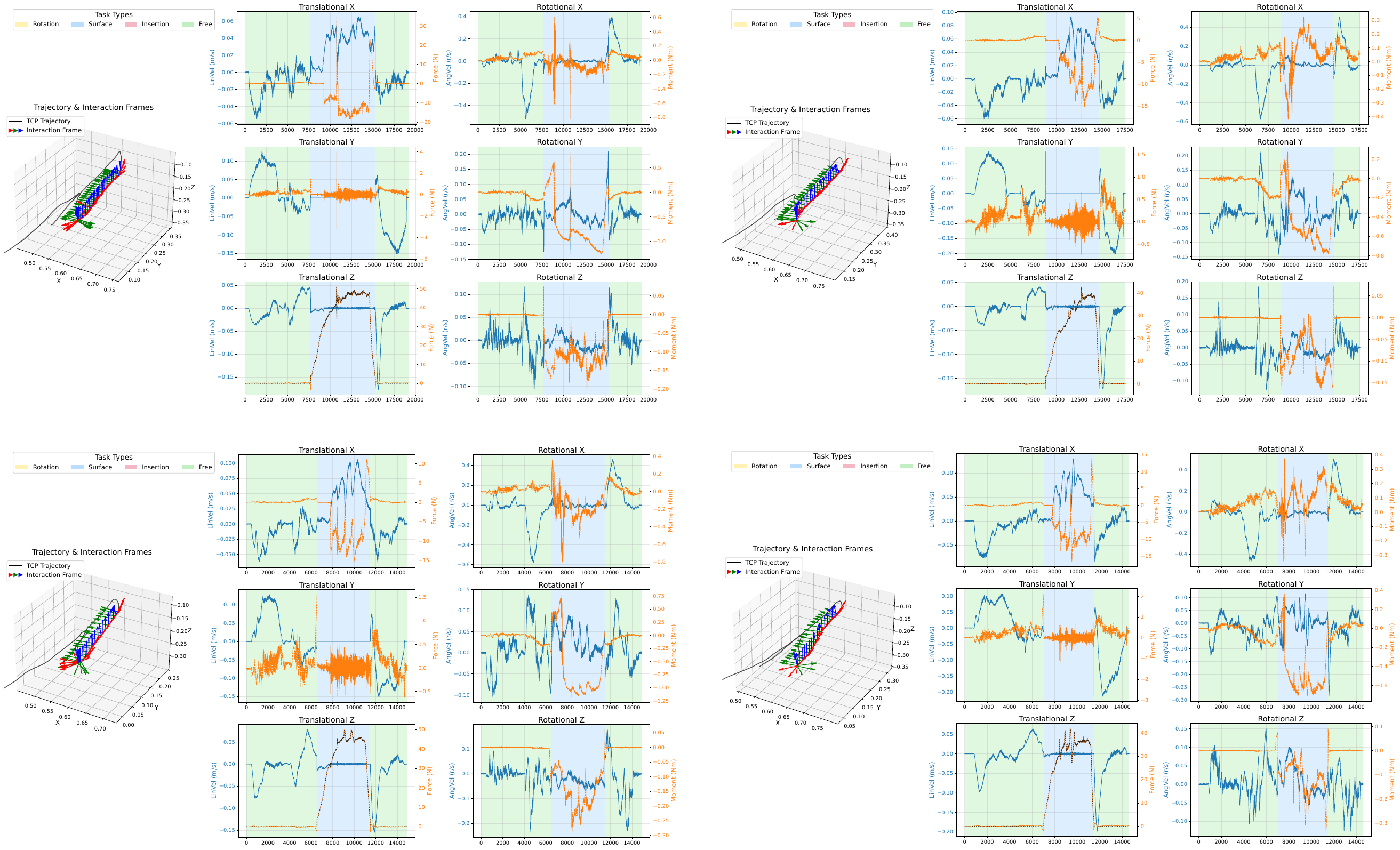}
    \caption{\textbf{Visualization of the Interaction Frame and the Task Mode on the \textit{Scrape off Sticker} Task.}}
    \label{fig:vis-scrape}\vspace{-0.3cm}
\end{figure*}

\subsection{High-Level Knowledge}

To distinguish between dissipative and structural residuals as the dominant power source, our method leverages high-level knowledge. Addressing potential doubts regarding the necessity of this prior information, we posit that twist and wrench signals are inherently ambiguous and cannot, on their own, reveal the nature of the power source.

\textbf{Justification.} From the perspective of sensor observation, the instantaneous power input $P = \boldsymbol{\mathcal W}^\top \boldsymbol{\xi}$ represents the energy transfer from the robot to the environment. However, this scalar value does not reveal the destination of the energy.
For instance, as shown in Fig.~\ref{fig:high-level}, a robot pushing a heavy block against friction (dissipative) and compressing a stiff spring (structural) can exhibit identical force-velocity profiles and positive work. Without high-level knowledge (e.g., object properties or task semantics) to characterize the environment's admittance, the raw signals cannot distinguish whether the energy is being dissipated as heat or stored as potential energy.

\begin{figure}[t]
    \centering
    \includegraphics[width=\linewidth]{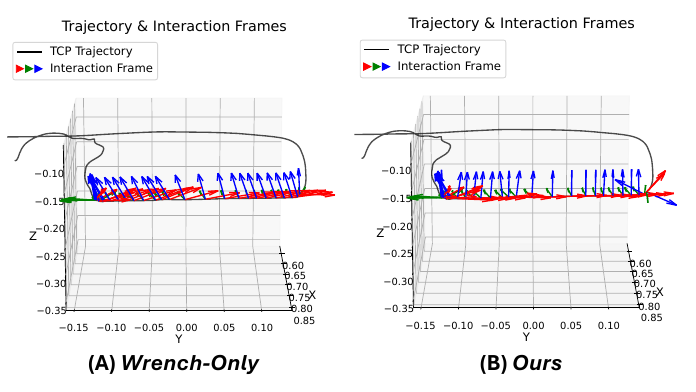}
    \caption{\textbf{Visualization of Different IF Recovery Methods on the \textit{Scrape Off Sticker} Task.} Wrench-only IF recovery methods couple friction directions, whereas our IF recovery method better aligns with the ground-truth (the world $z$-axis).}
    \label{fig:vis-cmp-scrape}\vspace{-0.4cm}
\end{figure}

\subsection{Visualization and Analysis}

To qualitatively validate the physical consistency of our approach, we visualize the recovered Interaction Frames (IF) across multiple expert demonstrations for three contact-rich tasks: \textit{\textbf{Push and Flip}} (Fig.~\ref{fig:vis-flip}), \textit{\textbf{Plug in EV Charger}} (Fig.~\ref{fig:vis-charger}), and \textit{\textbf{Scrape off Sticker}} (Fig.~\ref{fig:vis-scrape}). As illustrated, our IF recovery method consistently aligns the frame axes with the underlying task geometry across varying demonstrations, verifying that the recovered spectral axes are topologically stable.

A key advantage of our proposed method is its adaptivity to different types of task. In the \textit{Scrape off the Sticker} task (Fig.~\ref{fig:vis-cmp-scrape}), we explicitly compare our IF recovery method against a baseline constructed directly from raw interaction wrenches~\cite{forcemimic}, where the $z$-axis is aligned with the total measured force vector. In this task, significant tangential friction is generated during the scraping motion. For the wrench-based baseline, this friction acts as a parasitic wrench, causing the total force vector to deviate from the surface normal. Consequently, the baseline recovers a tilted frame that is coupled with the motion direction rather than the surface geometry, leading to substantially worse policy learning results.

\section{Force Policy}

In this section, we provide the detailed architecture and implementation specifications of our policy framework, comprising the global vision policy $\Pi_\text{global}$ and the local force policy $\Pi_\text{local}$.

\textbf{Global Vision Policy ($\Pi_\text{global}$).} We instantiate the global policy using RISE-2~\cite{rise2}. RISE-2 acts as a generalizable trajectory generator operating on global 3D point cloud observations. For our specific pipeline, we first train RISE-2 following its official implementation and freeze it during local force policy training. Let $I_{t'}$ denote the global observation at a low-frequency time step $t'$. The global policy processes this input to generate a latent action embedding, which is typically used to condition its internal diffusion head. In our framework, we intercept this embedding to serve as the global visual feature, denoted as $\phi(I_{t'}) \in \mathbb{R}^{512}$. This feature encapsulates high-level intent and geometry, guiding the local policy.

\textbf{Local Force Policy ($\Pi_\text{local}$).} The local policy $\Pi_\text{local}$ operates at a high frequency to handle contact-rich interactions. It consists of a multi-modal encoder, a feature fusion module, and a dual-head action decoder. The input to $\Pi_\text{local}$ includes a wrist-mounted camera image $I_t^{(w)} \in \mathbb{R}^{H \times W \times 3}$ and a history of end-effector poses $s_t \in \mathbb{R}^{T \times 9}$ (represented in continuous rotation~\cite{rot6d} format) and wrenches $\mathcal{W}_t \in \mathbb{R}^{T_o\times 6}$.

\begin{itemize}
\item \textbf{Vision Encoder.} We process the wrist image $I_t$ using a lightweight ResNet-18~\cite{resnet} backbone. To integrate the global context, we employ Feature-wise Linear Modulation (FiLM)~\cite{film}. The global feature $\phi(I_{t'})$ is projected to generate scale $\gamma$ and shift $\beta$ parameters, which modulate the intermediate feature maps of the ResNet:
$$
\text{FiLM}(x_i \mid \phi(I_{t'})) = \gamma_i(\phi(I_{t'})) \cdot x_i + \beta_i(\phi(I_{t'})),
$$
where $x_i$ represents the feature map at the $i$-th block. The final spatial features are pooled and projected to a visual embedding $z_{\text{vis}} \in \mathbb{R}^{128}$.

\item \textbf{Proprioceptive Encoder:} The proprioceptive history, consisting of end-effector poses and wrenches, is encoded using a Gated Recurrent Unit (GRU)~\cite{gru}. To incorporate the global context, we apply FiLM conditioning~\cite{film} to the linear projections at both the input and output of the GRU. Specifically, the global feature $\phi(I_{t'})$ modulates the proprioceptive features before they enter the recurrent unit and again after they exit. The final hidden state serves as the proprioceptive embedding $z_{\text{prop}} \in \mathbb{R}^{128}$.

\item \textbf{Adaptive Gated Fusion.} To effectively integrate the local sensory feedback with the high-level global intent, we employ a tri-modal adaptive gating mechanism. This allows the policy to dynamically weigh the importance of wrist vision, proprioception, and global context depending on the interaction phase. We first project the global feature $\phi(I_{t'})$ to the local embedding dimension via a linear layer $W_p$. The fusion weights and the final embedding are computed as:
$$\begin{aligned}
z_{\text{global}} &= W_p \phi(I_{t'}), \\
[\boldsymbol{\alpha}_1, \boldsymbol{\alpha}_2, \boldsymbol{\alpha}_3] &= \text{softmax}\left( W_g [z_{\text{vis}}; z_{\text{prop}}; z_{\text{global}}] + b_g \right), \\
z_{\text{fused}} &= \boldsymbol{\alpha}_1 \odot z_{\text{vis}} + \boldsymbol{\alpha}_2 \odot z_{\text{prop}} + \boldsymbol{\alpha}_3 \odot z_{\text{global}},
\end{aligned}$$
where the softmax is applied across the modality dimension, ensuring the gates sum to one.

\item \textbf{Interaction Structure Head.} Multi-layer perceptrons with three dense layers (sizes [128, 64, $D_{out}$]) are used to predict the task-oriented structural parameters. These head output: (1) the interaction frame pose relative to the end-effector pose $\Sigma_\text{IF} \in \mathbb{R}^9$ (also represented in continuous rotation format), (2) the reference wrench $\mathcal{W}_\text{ref} \in \mathbb{R}^6$, and (3) a binary selection mask $S\in\{0,1\}^6$ indicating the active control mode.

\item \textbf{Action Head.} We utilize a MIP head~\cite{mip} to generate precise control commands at high frequency. Conditioned on the fused feature $z_{\text{fused}}$, this head predicts a sequence of 50Hz robot actions in an action chunking~\cite{act} format to ensure temporal consistency. Specifically, at each time step $t$, the head outputs a chunk of $T_a$ future relative actions. At the execution frequency of 50Hz, the policy executes the first action of the chunk and discards the rest, following a receding horizon control strategy.
\end{itemize}

\textbf{Training.} 
The entire local force policy is trained end-to-end. We optimize the model using the AdamW optimizer with a learning rate of $10^{-4}$ and a cosine annealing schedule. The total loss function $\mathcal{L}$ is a weighted sum of the diffusion-based action loss and the auxiliary interaction structure losses:
$$
\mathcal{L} = \lambda_{\text{act}} \mathcal{L}_{\text{MIP}} + \lambda_{\text{frame}} \mathcal{L}_{\text{frame}} + \lambda_{\text{wrench}} \mathcal{L}_{\text{wrench}} + \lambda_{\text{mask}} \mathcal{L}_{\text{mask}},
$$
where $\mathcal{L}_{\text{MIP}}$ is the action prediction MSE loss for the MIP head, and $\mathcal{L}_{\text{frame}}$, $\mathcal{L}_{\text{wrench}}$, and $\mathcal{L}_{\text{mask}}$ correspond to the regression and classification losses for the interaction frame, reference wrench, and selection mask, respectively. We empirically set the weights $\lambda_\text{act} = \lambda_\text{frame} = \lambda_\text{wrench} = 1.0$ and $\lambda_\text{mask}=0.1$ to balance the task objectives.

Our local force policy can execute inference at 50Hz, and we let it output 50Hz control commands and directly takes over the controller if the router switches to fast.

\section{Dual-Policy Asynchronous Scheduler}

\subsection{Chunk Alignment with DTW}

Let $\mathbf{P} = \{p_0, p_1, \ldots, p_{N-1}\}$ denote the trajectory sequence predicted by the model, and $\mathbf{H} = \{h_{-M+1}, \ldots, h_0\}$ denote the robot's recently executed historical trajectory, where $h_0$ represents the current position. Our objective is to find the optimal starting index $k^*$ such that execution beginning from $p_{k^*}$ achieves a smooth transition with the current state.

We define the frame-wise cost function as:
\begin{equation}
\begin{aligned}
c(p_i, h_j) = & \; \; \; \;w_{\text{pos}} \|p_i^{\text{pos}} - h_j^{\text{pos}}\|\\ &+ w_{\text{ori}} d_{\text{ang}}(p_i^{\text{ori}}, h_j^{\text{ori}})\\ &+ w_\text{vel} \|\dot{p}_i - \dot{h}_j\|,
\end{aligned}
\end{equation}
where $w_{\text{pos}}$, $w_{\text{ori}}$, and $w_\text{vel}$ are the weights for position, orientation, and velocity, respectively, and $\dot{p}_i$ and $\dot{h}_j$ are velocities estimated via finite differences. We apply the Dynamic Time Warping (DTW) algorithm~\cite{dtw} to establish alignment between the beginning of the predicted trajectory and the end of the historical trajectory. By backtracking along the optimal path, we identify the optimal frame index $k^*$ corresponding to the current robot position $h_0$. Additionally, we compensate for the latency of the DTW computation itself:
\begin{equation}
k^* \leftarrow k^* + \left\lceil \frac{t_{\text{DTW}}}{\Delta t} \right\rceil.
\end{equation}
where $\Delta t$ denotes the scheduler execution period. In practice, the DTW procedure requires approximately 20ms to compute the optimal transition point, which corresponds to dropping an additional 1-2 steps at 50 Hz to compensate for DTW latency. Executing the trajectory from index $k^*$ ensures smooth continuity with the current state in both position and velocity, thereby avoiding abrupt transitions.

\subsection{Broader Impact on General Policies}

During experiments, we find that the asynchronous scheduler not only improves the smoothness of \textbf{\textit{Force Policy}}, but also enhances the performance of vision-only policies in contact-rich tasks through its waypoint dropout strategy and acceleration-continuous interpolation. In particular, the asynchronous scheduler leads to more stable contact behaviors for vision-only policies. We evaluate this effect using RISE-2~\cite{rise2} and $\pi_{0.5}$~\cite{pi05} as representative vision-only visuomotor and vision-language-action models, respectively. Their performance is compared against conventional sequential inference and execution with temporal ensemble~\cite{act}, as well as our asynchronous scheduler with DTW-based waypoint dropout during inference, on the \textbf{\textit{Push and Flip}} task.

\begin{table}[h]
\vspace{-0.1cm}
    \centering
    \begin{tabular}{ccrrr}
    \toprule
        \multirow{2}{*}{\textbf{Policy}} & \multirow{2}{*}{\textbf{w/ Scheduler?}} & \multicolumn{2}{c}{\textbf{Success Rate} $\uparrow$} & \multirow{2}{*}{\textbf{Avg.} $d$ (cm) $\downarrow$} \\ \cmidrule(lr){3-4}
        & & \multicolumn{1}{c}{push} & \multicolumn{1}{c}{flip} & \\
        \midrule
        \multirow{2}{*}{RISE-2~\cite{rise2}} & & \textbf{100.0}\% & 42.5\% & 0.86\\
         & \checkmark & \textbf{100.0}\% & \textbf{62.5}\% & \textbf{0.00} \\
        \midrule
        \multirow{2}{*}{$\pi_{0.5}$~\cite{pi05}} & & 80.0\% & 52.5\% & \textbf{0.00}\\
         & \checkmark & \textbf{95.0}\% & \textbf{77.5}\% & \textbf{0.00}\\
        \bottomrule
    \end{tabular}
    \caption{\textbf{Evaluation of Schedulers for Vision-Only Policies on the \textit{Push and Flip} Task.} Asynchronous scheduler improves the performance and make the contact more stable.}
    \label{tab:scheduler-eval}\vspace{-0.2cm}
\end{table}

The experimental results demonstrate a significant positive impact. Specifically, as shown in Table~\ref{tab:scheduler-eval}, the asynchronous scheduler consistently boosts the success rates and reduces control error across different policy architectures. For RISE-2, while the pushing success rate remains perfect, the flipping success rate sees a substantial improvement from 42.5\% to 62.5\%, and the average displacement error drops to zero. Similarly, for the VLA model $\pi_{0.5}$, applying our scheduler yields improvements in both sub-tasks, raising the flipping success rate to 77.5\%. 

We attribute these gains primarily to the \textbf{DTW-based waypoint dropout mechanism} in preventing contact loss. Vision-based policies often suffer from temporal inconsistencies between action chunks~\cite{sail,rtc}, occasionally predicting erroneous ``retreating'' motions (backward fluctuations relative to the task progress). In contact-rich manipulation like flipping, such retreating actions can cause the end-effector to momentarily detach from the object surface, leading to a complete loss of the established contact state and subsequent failure. By aligning the predicted chunk with the current execution progress via DTW and dropping these inconsistent preceding waypoints, our scheduler ensures a monotonic and continuous interaction profile, thereby maintaining stable contact throughout the maneuver. This highlights the potential of our scheduler as a plug-and-play module for enhancing the physical robustness of general visuomotor policies.

\section{Experiment Details}

\subsection{Task Design}

We carefully designed an evaluation suite comprising three distinct tasks that span two fundamental categories of contact-rich manipulation: \textbf{surface polishing} and \textbf{peg-in-hole insertion}. While many prior works limit their evaluation to a single interaction category, our benchmark covers both continuous surface tracking and geometric constrained assembly. Furthermore, we upgraded the task designs to elevate their difficulty, thereby offering a more rigorous assessment of each policy's capabilities. The task descriptions are shown as follows.

\begin{itemize}
    \item \textbf{\textit{Push and Flip}}. Inspired by~\cite{kober2015learning} but with slight modifications, this task requires the robot to push an object until it contacts a wall (formed by a large heavy object), and then flip the box upright using both the ground and the wall. The task poses 3 key challenges: \textbf{(1)} correctly sensing contact and switching to the flipping phase; switching too early loses contact, while switching too late pushes the heavy object; \textbf{(2)} exerting forces along changing directions during flipping; and \textbf{(3)} avoiding applying excessive forces that would push the heavy object.

    \item \textbf{\textit{Plug in EV Charger}}. Inspired by~\cite{tavla} but with modifications, this task requires the robot to plug an EV charger into a socket. Unlike the original task in~\cite{tavla}, our fast-charging socket has tighter tolerances and requires sufficiently large insertion forces (approximately 160N) to achieve full engagement. The task poses 3 key challenges: \textbf{(1)} precisely aligning the plug with the socket under tight tolerances; \textbf{(2)} sensing incorrect contact with the socket surface and adjusting the insertion direction based on force feedback; and \textbf{(3)} exerting sufficiently large forces to complete insertion without causing jamming.

    \item \textbf{\textit{Scrape off Sticker}}. This task is self-designed and requires the robot to scrape stickers off a surface. It is divided into an easy and a hard version. In the easy version, the robot only needs to scrape the side without glue, so merely establishing contact is sufficient. In the hard version, the robot must scrape the side with glue, requiring the application of sufficient force (approximately 35N) to remove the sticker. The task poses 3 key challenges: \textbf{(1)} maintaining stable contact while moving along the surface, \textbf{(2)} sensing and adjusting forces to successfully remove stickers with glue, and \textbf{(3)} avoiding excessive forces that could damage the surface or the object.

\end{itemize}

\subsection{Data Collection}

Recently, a growing body of teleoperation systems incorporate haptic feedback~\cite{rdp, zhou2025admittance, kamijo2024learning, bazhenov2025echo, satsevich2025prometheus}. However, many implementations provide haptics through visual cues or vibrotactile signals, rather than delivering physically faithful force feedback. Alternatives include kinesthetic teaching~\cite{acp}, handheld interfaces~\cite{umi, forcemimic, choi2026wild}, and exoskeleton-based systems~\cite{airexo, rise2, dexop}. For contact-rich manipulation, high-quality demonstrations typically require both accurate robot state sensing and a global-view camera to support the global vision policy. We therefore adopt arm-to-arm teleoperation~\cite{tdk}, which mirrors the interaction forces measured on the follower arm back to the operator's controller arm. This design enables high-fidelity force feedback and substantially improves demonstration quality. For example, in \textbf{\textit{Push and Flip}}, the operator maintains a contact force of roughly 10N-20N to prevent overloading; in \textbf{\textit{Scrape off Sticker}}, the operator regulates the contact force within 35N--50N to fully remove the sticker. Then, we collect 50 high-quality demonstrations for each task as the training data. The RGB-D images are recorded at 15Hz, and the robot states (including end-effector pose, end-effector velocity, and force/torque signals) are recorded at 1000Hz.
\subsection{Baselines}

In this work, we compare our method against several baselines, including vision-only visuomotor policies~\cite{rise2}, vision-language-action (VLA) models~\cite{pi05}, force-aware visuomotor policies~\cite{foar, rdp}, and force-aware VLA models~\cite{forcevla, tavla}.

\begin{itemize}

\item \textbf{\textit{RISE-2}}~\cite{rise2}, a generalizable visuomotor policy that fuses sparse 3D geometric features with dense 2D semantic features for action prediction.

\item $\boldsymbol{\pi}_{0.5}$~\cite{pi05}, a VLA model exhibiting open-world generalization capabilities, co-trained on heterogeneous tasks across large-scale robotic datasets~\cite{oxe, rt1, rh20t}.

\item \textbf{\textit{FoAR}}~\cite{foar}, a force-aware visuomotor policy based on RISE~\cite{rise} that selectively integrates force features utilizing a future contact predictor.

\item \textbf{\textit{RDP}}~\cite{rdp}, a slow-fast force-aware reactive policy based on Diffusion Policy (DP)~\cite{dp}. It uses a fast force-aware tokenizer for latent action encoding/decoding and utilizes a latent diffusion policy to predict latent action chunks.

\item \textbf{\textit{ForceVLA}}~\cite{forcevla}, a force-aware VLA model based on $\pi_0$~\cite{pi0} that incorporates a mixture-of-experts (MoE) architecture~\cite{moe} for vision-language-force fusion.

\item \textbf{\textit{TA-VLA}}~\cite{tavla}, a force-aware VLA model based on $\pi_0$~\cite{pi0} that leverages observed torque sequences as input and predicts future torque as an auxiliary supervision.
\end{itemize}

To ensure a fair comparison, we utilize 6-DoF force/torque signals measured at the end-effector as the force input/output for all baselines. All other hyperparameters adhere to the specifications in their respective official implementations.

\subsection{Metrics}

We use stage-wise, accumulated success rates as the primary metric for each task. The partial successes (recorded as 0.5) for the ``flip'' stage of \textbf{\textit{Push and Flip}} and the ``plug in'' stage of \textbf{\textit{Plug in EV Charger}} are demonstrated in Fig.~\ref{fig:metric}.

\begin{figure}[h]
    \centering
    \includegraphics[width=0.8\linewidth]{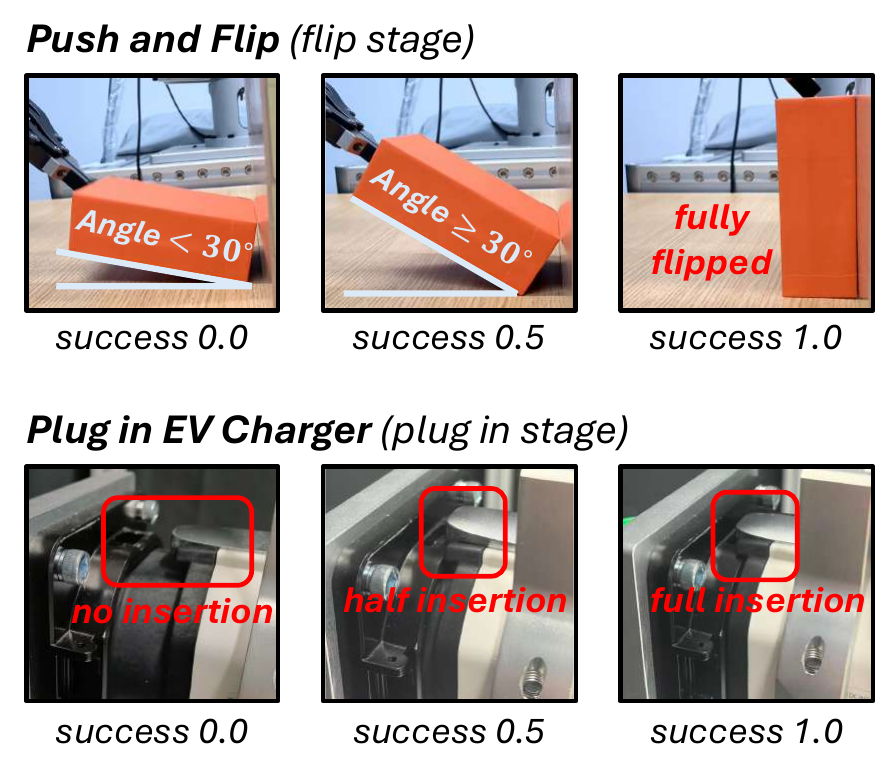}
    \caption{\textbf{Evaluation Metrics Explanation of Partial Success.}}
    \label{fig:metric}
\end{figure}

\begin{figure*}
    \centering
    \includegraphics[width=\linewidth]{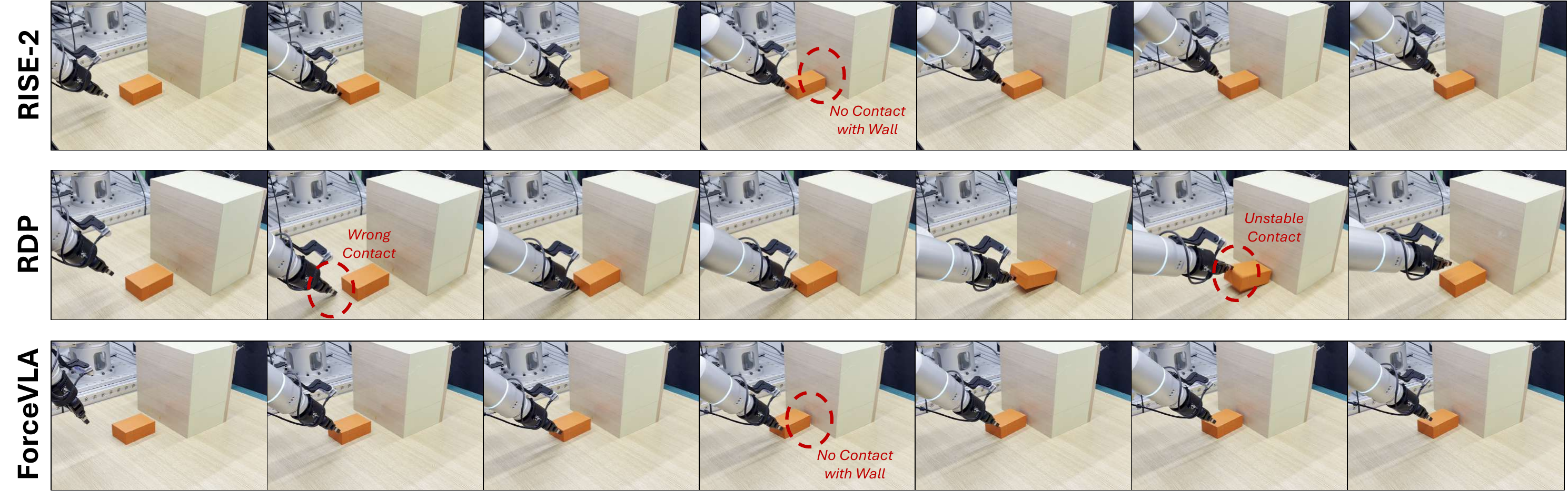}
    \caption{\textbf{Failure Analyses on the \textit{Push and Flip} Task.} Most failures arise from attempting to flip the object without first establishing contact with the wall. Force-aware policies using position control often exhibit unstable or unintended contacts.}
    \label{fig:failure-flip}
\end{figure*}

\begin{figure*}
    \centering
    \includegraphics[width=\linewidth]{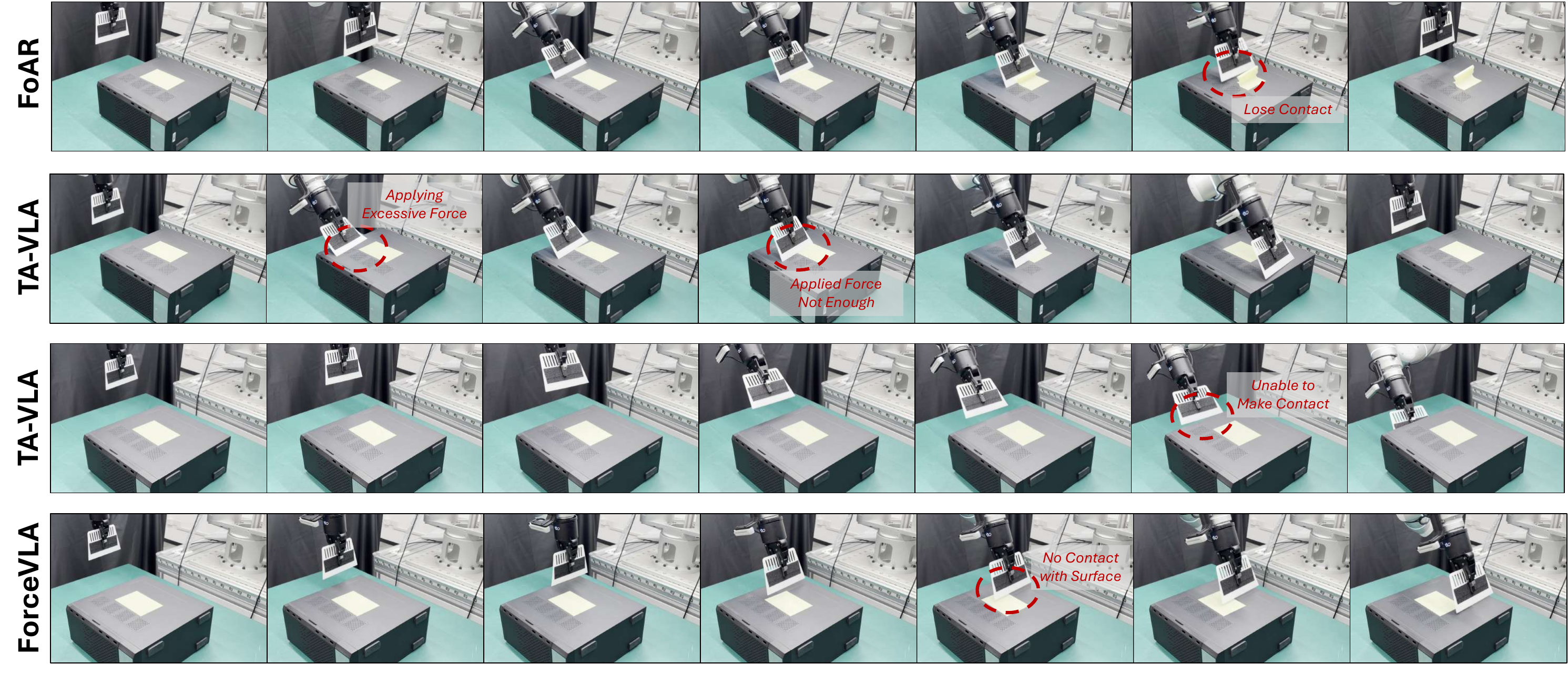}
    \caption{\textbf{Failure Analyses on the \textit{Scrape off Sticker} (Hard) Task.} Force-aware VLA models typically fail to establish stable contact with the surface. All baselines usually cannot exert sufficient and stable force to scrape off the sticker.}
    \label{fig:failure-shovel}\vspace{-0.3cm}
\end{figure*}

\subsection{Failure Analyses}

We summarize typical failure cases on \textbf{\textit{Push and Flip}} in Fig.~\ref{fig:failure-flip}. Most failures occur when the policy attempts to flip the object before establishing contact with the wall. This issue is especially pronounced for vision-only policies due to the absence of force feedback. Notably, many force-aware baselines exhibit the same failure mode: it is difficult to learn the critical ``contact-then-transition'' logic from demonstrations alone, as the transition happens within roughly 50ms in human executions. RDP shows an additional failure pattern. When its predicted latent action chunk corresponds to pushing, but it unexpectedly pushes the box to reach the wall, the resulting out-of-distribution force signals can destabilize the fast tokenizer, producing erratic actions that induce incorrect rotations and ultimately lead to failure. Finally, force-aware methods that rely on position control often struggle to maintain stable contact during the flip, causing the box to slip or drop and preventing task completion. On the contrary, our \textbf{\textit{Force Policy}} utilizes force control to establish and maintain contact, effectively mitigating the issues that occur in baselines.

\begin{figure*}
    \centering
    \includegraphics[width=\linewidth]{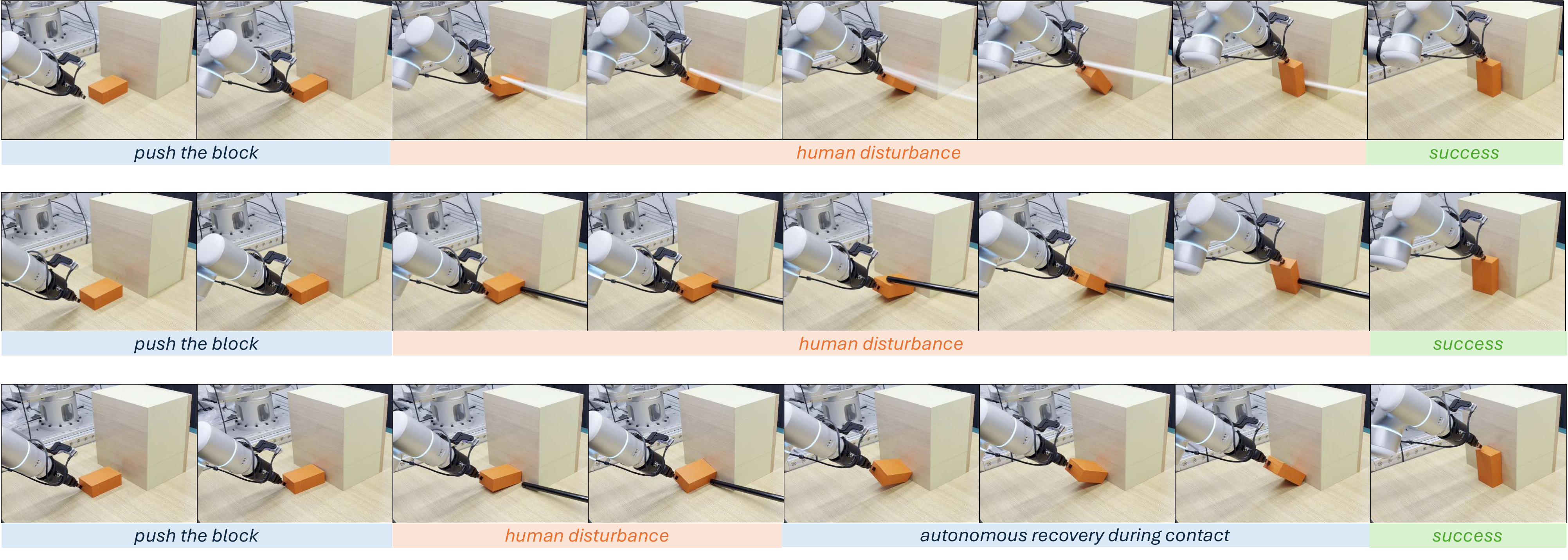}
    \caption{\textbf{Contact Robustness of \textit{Force Policy} under Disturbances in the \textit{Push and Flip} Task.} \textbf{\textit{Force Policy}} is robust under human disturbances and can perform autonomous recovery during the contact phase.}
    \label{fig:disturbance}\vspace{-0.3cm}
\end{figure*}

For \textbf{\textit{Plug in EV Charger}}, most approaches fail to apply sufficient force to fully seat the connector in the socket. This task is particularly challenging because the required insertion force is close to the robotic arm's torque limits, and small motion errors or interaction frame estimation inaccuracies can trigger the hardware protection mechanisms and terminate execution. Our \textbf{\textit{Force Policy}} can also encounter this failure mode. Addressing it likely requires finer-grained, more accurate force regulation (especially near saturation), which we leave as an important direction for future work.

Typical failure cases on \textbf{\textit{Scrape off Sticker}} (Hard) are shown in Fig.~\ref{fig:failure-shovel}. Force-aware VLA models such as TA-VLA and ForceVLA often fail to establish contact with the surface, suggesting that noisy force measurements can hinder vision-guided global motion, which is consistent with the observation in~\cite{foar}. Across baselines, another common failure mode is insufficient and unstable contact force: policies frequently lose contact and revert to smaller forces, leaving the sticker partially intact. In contrast, \textbf{\textit{Force Policy}} leverages explicit force regulation, using predicted control parameters to maintain the desired contact force in the desired direction, which yields a substantial performance improvement.

\subsection{Generalization Evaluation}

To rigorously evaluate the robustness of our learned policy, we curated a diverse set of unseen test objects for the \textit{Push and Flip} task. These objects were specifically selected to introduce significant domain shifts relative to the training distribution, probing the policy's ability to adapt to variations in physical properties and visual appearance. A complete catalog of these objects is visualized in Fig.~\ref{fig:gen-object}. The test set is categorized along three primary axes of variation:

\begin{itemize} 
\item \textbf{Geometric Variation.} The objects feature distinct geometric profiles, ranging from standard cuboids to irregular shapes and cylinders (\textit{e.g.}, cylindrical cup). Variations in aspect ratio and edge curvature challenge the policy's ability to locate stable contact points and execute precise flipping motions without prior shape knowledge.

\item \textbf{Stiffness and Material Compliance.} Variations in object stiffness are critical for evaluating physical adaptation for force-aware policies. The test set includes rigid objects like the hard wooden block and deformable objects like the soft sponge. These materials have different force-deformation properties, requiring the policy to dynamically adjust its commands to maintain effective interaction without crushing soft objects or slipping on hard ones.

\item \textbf{Visual Appearance.} To test visual generalization, we included objects with colors and textures that were not present in the training dataset. This includes objects with complex textures (\textit{e.g.}, square box) and different colors (\textit{e.g.}, purple box), ensuring that the policy is robust to perceptual noise.
\end{itemize}

\begin{figure}[h]
    \centering
    \includegraphics[width=0.75\linewidth]{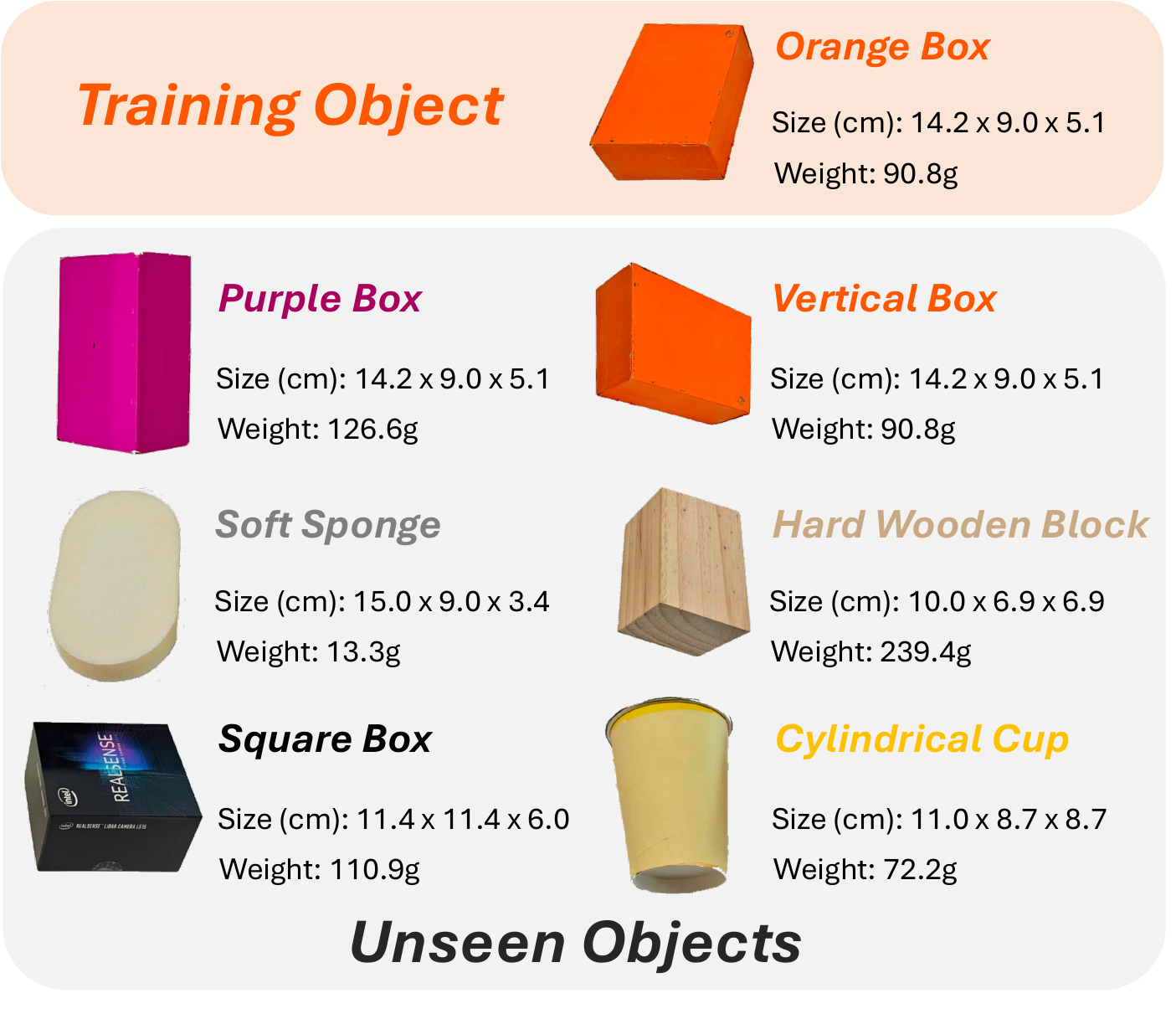}
    \caption{\textbf{Objects in Generalization Evaluation.} Unseen objects of different colors, geometries, and stiffnesses are selected to evaluate the generalization ability of the policies.}
    \label{fig:gen-object}\vspace{-0.3cm}
\end{figure}

As demonstrated in the experiments, our method successfully generalizes across these categories, attributing its success to the explicit modeling of the interaction frame and the adaptive fusion of proprioceptive feedback, which compensates for visual ambiguities and geometric uncertainties.

\begin{figure*}
    \centering
    \includegraphics[width=\linewidth]{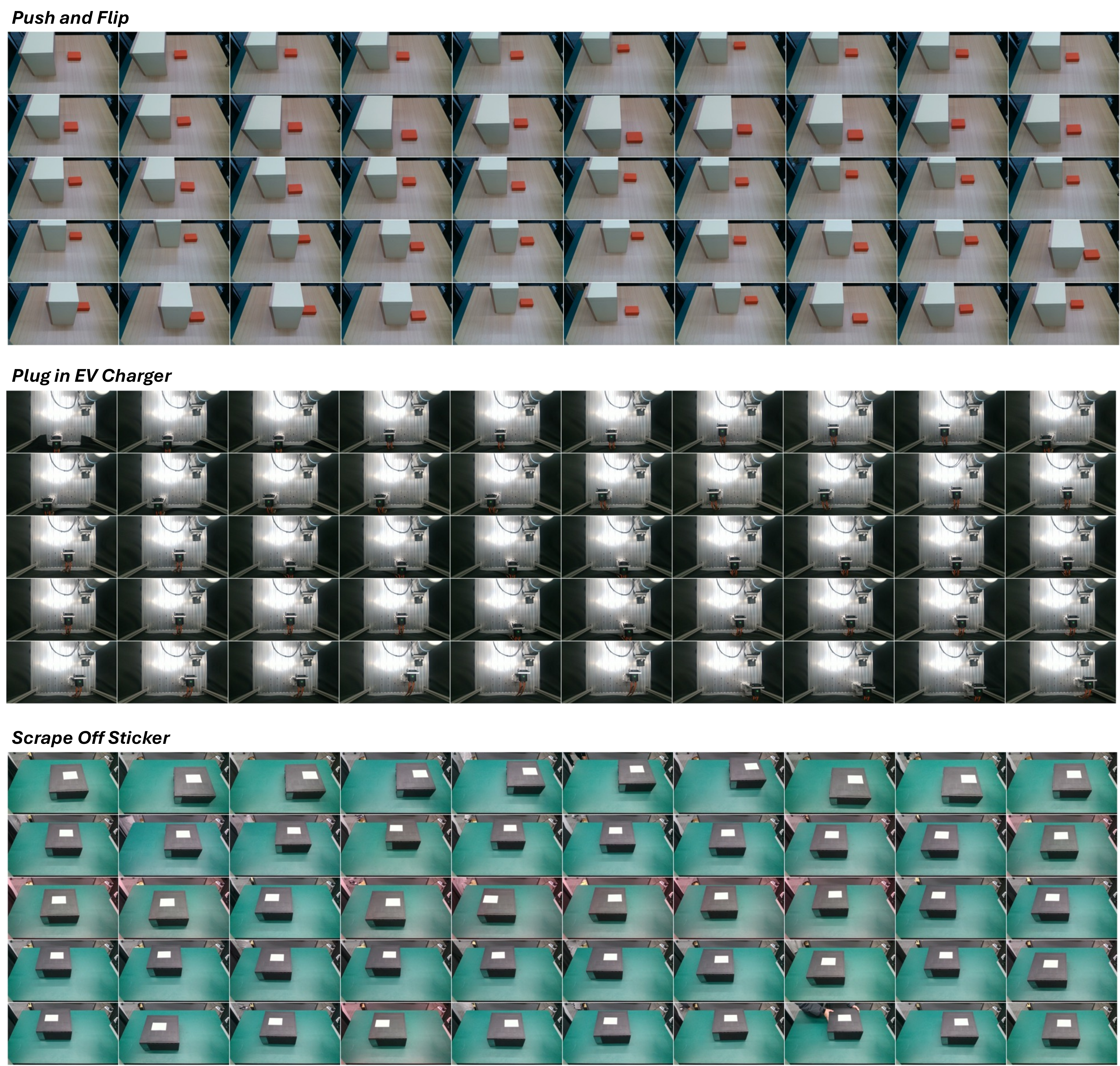}
    \caption{\textbf{Demonstration Overview for Each Task.} We visualize the initial configurations from 50 demonstrations for each task.}
    \label{fig:demo-scene}
\end{figure*}

\begin{figure*}
    \centering
    \includegraphics[width=\linewidth]{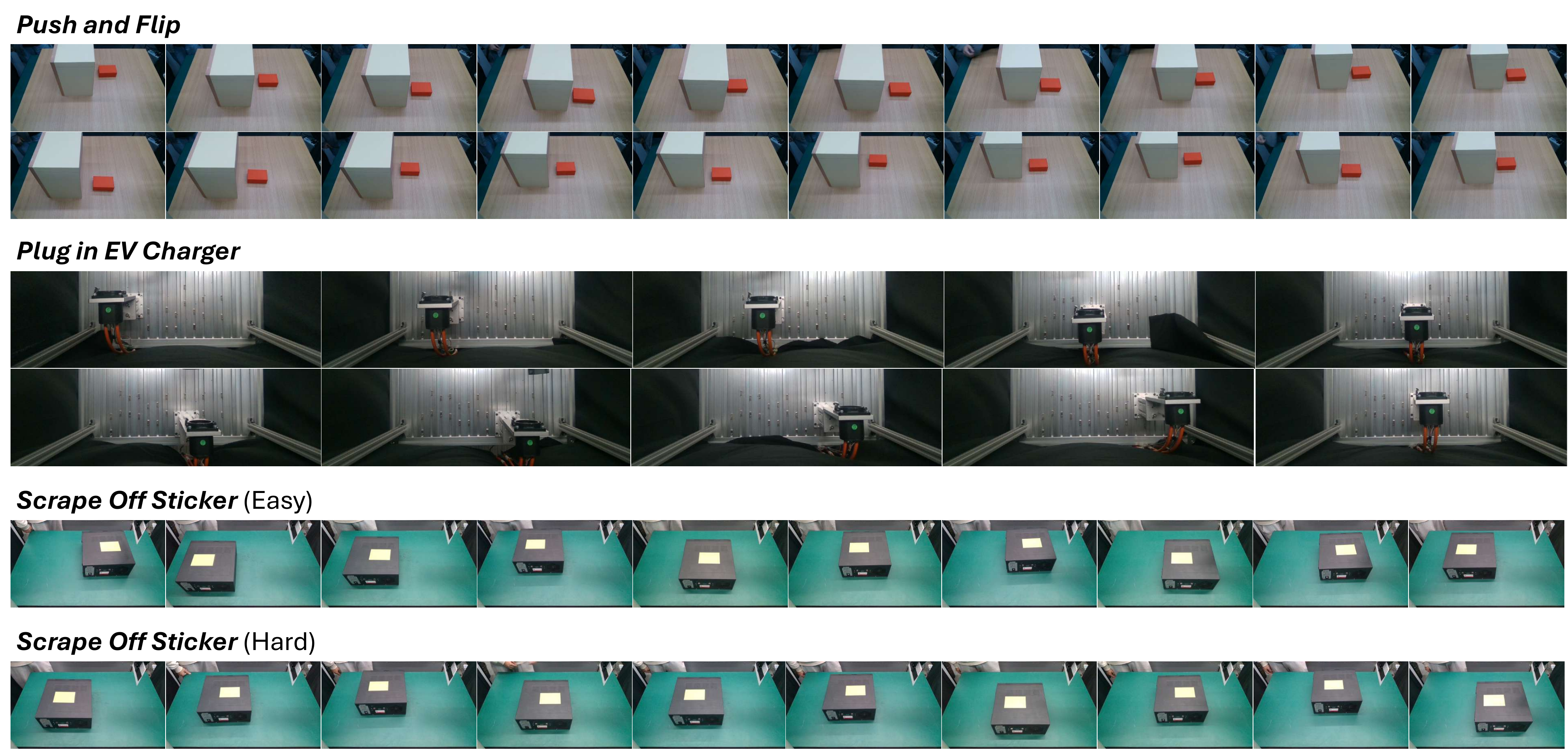}
    \caption{\textbf{Evaluation Configurations for Each Task.} These configurations are randomly initialized beforehand.}
    \label{fig:eval-scene}
\end{figure*}

\subsection{Contact Robustness}

To evaluate the policy's stability in unstructured physical interactions, we introduced unexpected human disturbances during the contact maintenance phase in the \textbf{\textit{Push and Flip}} task. As depicted in Fig.~\ref{fig:disturbance}, the proposed Force Policy demonstrates significant robustness against unmodeled disturbances. A key characteristic of the policy is the explicit regulation of the interaction wrench. When the human operator physically interferes with the robot, the policy adapts the end-effector pose to maintain the target wrench $\boldsymbol{\mathcal{W}}^*$. 

This mechanism is critical for ensuring contact stability. \textbf{\textit{Force Policy}} prioritizes the stability of the contact force over geometric strictness in an adaptive manner. By dynamically adjusting the robot's pose to satisfy the force constraint, the system effectively absorbs the disturbance. Consequently, the contact force remains smooth and constant despite external disturbances, verifying the policy's ability to maintain stable physical interaction under uncertainty.

\subsection{Demonstration Overview}

We visualize the initial configurations of all 50 demonstrations for each task in Fig.~\ref{fig:demo-scene}. The objects are randomly placed within the workspace during demonstration collection.

\subsection{Evaluation Overview}

We visualize the evaluation configurations for each task in Fig.~\ref{fig:eval-scene}. These configurations are randomly initialized before any policy is trained or tested, following previous practices in~\cite{cage, umi} for fair comparisons.

\subsection{Results with Confidence Interval}

Final-stage results with 95\% confidence intervals are reported in Tab.~\ref{tab:result_final_only_compact}. The results show that \textbf{\textit{Force Policy}} significantly outperforms baselines.

\begin{table}[htbp]
    \centering
    \caption{\textbf{Final-Stage Success Rates in Evaluations}.}
    \label{tab:result_final_only_compact}
    \footnotesize
    \setlength{\tabcolsep}{3pt}
    \begin{tabular}{l cccc}
    \toprule
        \textbf{Policy} 
         & \textbf{\textit{Flip}} 
         & \textbf{\textit{Charger}} 
         & \textbf{\textit{Scrape}} (Easy)
         & \textbf{\textit{Scrape}} (Hard) \\
    \midrule
         RISE-2      & 42.5 $\pm$ 21.8 & 0.0 $\pm$ 0.0 & 80.0 $\pm$ 30.2 & 10.0 $\pm$ 22.6 \\
         $\pi_{0.5}$ & 52.5 $\pm$ 17.8 & 0.0 $\pm$ 0.0 & 70.0 $\pm$ 34.6 & 20.0 $\pm$ 30.2 \\ \midrule
         RDP         & 57.5 $\pm$ 19.0 & 5.0 $\pm$ 11.3 & 70.0 $\pm$ 34.6 & 20.0 $\pm$ 30.2 \\
         FoAR        & 60.0 $\pm$ 23.5 & 10.0 $\pm$ 15.1 & 40.0 $\pm$ 36.9 & 20.0 $\pm$ 30.2 \\
         ForceVLA    & 30.0 $\pm$ 19.2 & 0.0 $\pm$ 0.0 & 0.0 $\pm$ 0.0 & 0.0 $\pm$ 0.0 \\
         TA-VLA      & 62.5 $\pm$ 22.6 & 0.0 $\pm$ 0.0 & 50.0 $\pm$ 37.7 & 10.0 $\pm$ 22.6 \\
    \midrule
         \rowcolor[HTML]{f2f2f2} Force Policy 
         & \textbf{95.0 $\pm$ 7.2} 
         & \textbf{65.0 $\pm$ 24.1} 
         & \textbf{100.0 $\pm$ 0.0} 
         & \textbf{90.0 $\pm$ 22.6} \\
    \bottomrule
    \multicolumn{5}{l}{\footnotesize \textbf{Note}: All values in \%. $\pm$ denotes 95\% CI. Clip to [0, 100]\%.}
    \end{tabular}
\end{table}

\end{document}